\documentclass{article}

\usepackage[utf8]{inputenc} %
\usepackage[T1]{fontenc}    %
\usepackage{hyperref}       %
\usepackage{url}            %
\usepackage{booktabs}       %
\usepackage{amsfonts}       %
\usepackage{nicefrac}       %
\usepackage{microtype}      %
\usepackage{xcolor}         %

\usepackage{amsmath,amsfonts,amsthm}       %
\usepackage{subcaption}
\usepackage{bbm}
\usepackage{multirow}
\usepackage{enumitem}
\usepackage{diagbox}
\usepackage{pifont}
\usepackage[capitalise]{cleveref}  %
\usepackage{comment}
\usepackage{etoolbox}
\usepackage{graphicx}
\usepackage{wrapfig}
\usepackage[font=small]{caption}

\usepackage{pifont}%
\newcommand{\xmark}{\ding{55}}%

\usepackage{import}
\graphicspath{{figures/}{../figures/}}
\usepackage{subfiles}

\newtoggle{arxiv}
\toggletrue{arxiv}

\newtheorem{theorem}{Theorem}
\newtheorem{lemma}{Lemma}[section]
\newtheorem{corollary}[lemma]{Corollary}

\newtheorem{proposition}[theorem]{Proposition}
\newtheorem{definition}{Definition}
\newtheorem{remark}[lemma]{Remark}

\newcommand{\dt}{\Delta t}
\newcommand{\dd}{\mathop{}\!d}
\DeclareMathOperator{\hippo}{\mathsf{hippo}}
\DeclareMathOperator*{\diag}{diag}

\newcommand{\F}{\mathbb{F}}
\newcommand{\R}{\mathbb{R}}
\newcommand{\calG}{\mathcal G}
\newcommand{\parens}[1]{\left( {#1}\right)}
\newcommand{\set}[1]{\left\{ #1 \right\}}

\newcommand{\vA}{\mathbf{A}}
\newcommand{\vB}{\mathbf{B}}
\newcommand{\vC}{\mathbf{C}}
\newcommand{\vD}{\mathbf{D}}
\newcommand{\vE}{\mathbf{E}}
\newcommand{\vI}{\mathbf{I}}
\newcommand{\vM}{\mathbf{M}}

\newcommand{\vP}{\mathbf{P}}
\newcommand{\vQ}{\mathbf{Q}}
\newcommand{\vR}{\mathbf{R}}
\newcommand{\vS}{\mathbf{S}}
\newcommand{\vU}{\mathbf{U}}
\newcommand{\vV}{\mathbf{V}}

\newcommand{\vb}{\mathbf{b}}

\newcommand{\vx}{\mathbf{x}}

\newcommand{\vz}{\mathbf{z}}
\newcommand{\vzero}{\mathbf{0}}

\newcommand{\jab}[1]{{J_{#1}^{\alpha,\beta}}}
\newcommand{\japob}[2]{{J_{#1}^{\alpha#2,\beta}}}
\newcommand{\jabpo}[2]{{J_{#1}^{\alpha,\beta#2}}}
\newcommand{\jabp}[2]{{J_{#1}^{\alpha#2,\beta#2}}}
\newcommand{\pab}[1]{{p_{#1}^{\alpha,\beta}}}

\newcommand{\x}{{Y}}

\iftoggle{arxiv}{
    \usepackage[numbers,sort]{natbib}
    \setlength{\textwidth}{6.5in}
    \setlength{\textheight}{9in}
    \setlength{\oddsidemargin}{0in}
    \setlength{\evensidemargin}{0in}
    \setlength{\topmargin}{-0.5in}
    \newlength{\defbaselineskip}
    \setlength{\defbaselineskip}{\baselineskip}
    \setlength{\marginparwidth}{0.8in}
}{
\PassOptionsToPackage{square,numbers,sort}{natbib}
\usepackage[final]{neurips_2021}
}

\iftoggle{arxiv}{
  \title{Combining Recurrent, Convolutional, and Continuous-time Models with Linear State-Space Layers}
  \usepackage{authblk}
  \author[$\dagger$]{Albert Gu}
  \author[$\dagger$]{Isys Johnson}
  \author[$\dagger$]{Karan Goel}
  \author[$\dagger$]{Khaled Saab}
  \author[$\dagger$]{Tri Dao}
  \author[$\ddagger$]{Atri Rudra}
  \author[$\dagger$]{Christopher R{\'e}}
  \affil[$\dagger$]{Department of Computer Science, Stanford University}
  \affil[$\ddagger$]{Department of Computer Science and Engineering, University at Buffalo, SUNY\vspace{4pt}}
  \affil[ ]{{\texttt{\{albertgu,knrg,ksaab,trid\}@stanford.edu}, \texttt{\{isysjohn,atri\}@buffalo.edu}, \texttt{chrismre@cs.stanford.edu}}}
}{
\title{Combining Recurrent, Convolutional, \\ and Continuous-time Models with \\ Linear State-Space Layers}

\author{%
\hspace{-1.3em}
Albert Gu$^\dagger$,
Isys Johnson$^\ddagger$,
Karan Goel$^\dagger$,
Khaled Saab$^\ast$,
Tri Dao$^\dagger$,
Atri Rudra$^\ddagger$,
Christopher R{\'e}$^\dagger$
\\
\hspace{-2.3em}$^\dagger$ Department of Computer Science, Stanford University\\
\hspace{-2.3em}$^\ast$ Department of Electrical Engineering, Stanford University\\
\hspace{-2.3em}$^\ddagger$ Department of Computer Science and Engineering, University at Buffalo, SUNY\\
{\small\texttt{\{albertgu,knrg,ksaab,trid\}@stanford.edu}, \texttt{chrismre@cs.stanford.edu}}
\\
{\small\texttt{\{isysjohn,atri\}@buffalo.edu}}
}
}

\begin{document}

\maketitle

\begin{abstract}
  Recurrent neural networks (RNNs), temporal convolutions, and neural differential equations (NDEs) are popular families of deep learning models for time-series data,
  each with unique strengths and tradeoffs in modeling power and computational efficiency.
  We introduce a simple sequence model inspired by control systems that generalizes these approaches while addressing their shortcomings.
  The Linear State-Space Layer (LSSL) maps a sequence $u \mapsto y$ by simply simulating a linear continuous-time state-space representation $\dot{x} = Ax + Bu, y = Cx + Du$.
  Theoretically, we show that LSSL models are closely related to the three aforementioned families of models and inherit their strengths.
  For example, they generalize convolutions to continuous-time, explain common RNN heuristics,
  and share features of NDEs such as time-scale adaptation.
  We then incorporate and generalize recent theory on continuous-time memorization to introduce a trainable subset of structured matrices $A$ that endow LSSLs with long-range memory.
  Empirically, stacking LSSL layers into a simple deep neural network obtains state-of-the-art results across time series benchmarks for long dependencies in sequential image classification, real-world healthcare regression tasks, and speech.
  On a difficult speech classification task with length-16000 sequences, LSSL outperforms prior approaches by 24 accuracy points, and even outperforms
baselines that use hand-crafted features on 100x shorter sequences.
\end{abstract}

\section{Introduction}
\label{sec:intro}

A longstanding challenge in machine learning is efficiently modeling sequential data longer than a few thousand time steps.
The usual paradigms for designing sequence models involve recurrence (e.g. RNNs), convolutions (e.g. CNNs), or differential equations (e.g. NDEs),
which each come with tradeoffs.
For example, RNNs are a natural stateful model for sequential data that require only constant computation/storage per time step, but are slow to train and suffer from optimization difficulties (e.g., the "vanishing gradient problem" \citep{pascanu2013difficulty}), which empirically limits their ability to handle long sequences.
CNNs encode local context and enjoy fast, parallelizable training, but are not sequential, resulting in more expensive inference and an inherent limitation on the context length.
NDEs are a principled mathematical model that can theoretically address continuous-time problems and long-term dependencies \citep{morrill2021neural}, but are
very inefficient.

Ideally, a model family would combine the strengths of these paradigms, providing properties like parallelizable training (convolutional), stateful inference (recurrence) and time-scale adaptation (differential equations), while handling very long sequences in a computationally efficient way.
Several recent works have turned to this question.
These include the CKConv, which models a continuous convolution kernel~\citep{romero2021ckconv}; several ODE-inspired RNNs, such as the UnICORNN~\citep{rusch2021unicornn};
the LMU, which speeds up
a specific linear recurrence using convolutions~\citep{voelker2019legendre,chilkuri2021parallelizing};
and HiPPO~\citep{gu2020hippo}, a generalization of the LMU that introduces a theoretical framework for continuous-time memorization.
However, these model families come at the price of reduced \emph{expressivity}: 
intuitively, a family that is both convolutional and recurrent should be more restrictive than either.

\begin{figure}[!t]
    \centering
    \includegraphics[width=\linewidth]{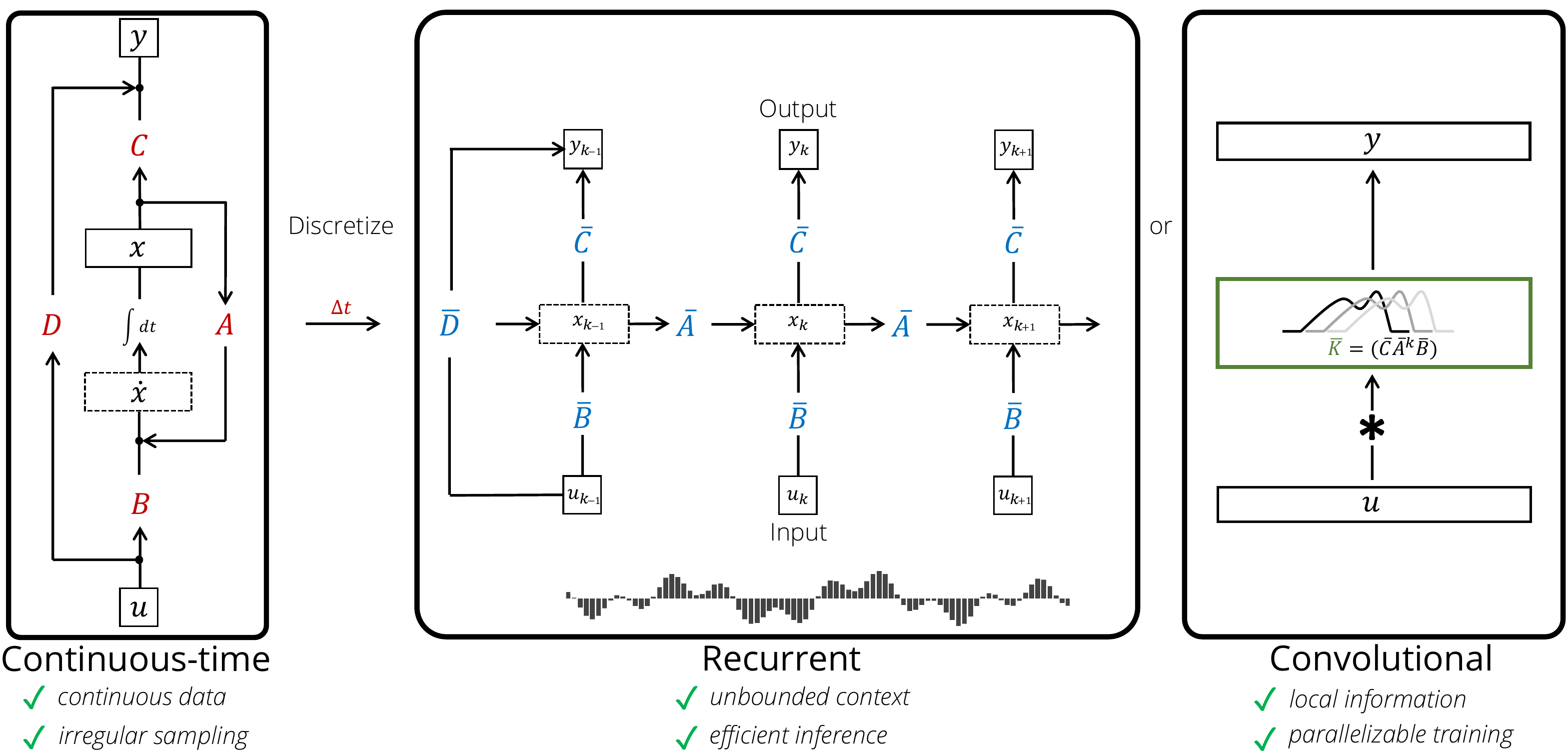}
    \caption{(\textbf{Three views of the LSSL})
      A \textbf{Linear State Space Layer} layer is a map \( u_t \in \mathbb{R} \to y_t \in \mathbb{R} \), where each feature \( u_t \mapsto y_t \) is defined by discretizing a state-space model \( A, B, C, D \) with a parameter \( \dt \). %
  The underlying state space model defines a discrete recurrence through combining the state matrix \( A \) and timescale \( \dt \) into a transition matrix \( \overline{A} \).
  (\textbf{Left})
  As an implicit continuous model, irregularly-spaced data can be handled by discretizing the same matrix \( A \) using a different timescale \( \dt \).
  (\textbf{Center}) As a recurrent model, inference can be performed efficiently by computing the layer \emph{timewise} (i.e., one vertical slice at a time \( (u_t, x_t, y_t), (u_{t+1}, x_{t+1}, y_{t+1}), \dots \)), by unrolling the linear recurrence.
  (\textbf{Right}) As a convolutional model, training can be performed efficiently by computing the layer \emph{depthwise} in parallel (i.e., one horizontal slice at a time \( (u_t)_{t \in [L]}, (y_t)_{t \in [L]}, \dots \)), by convolving with a particular filter.
}
\label{fig:splash}
\end{figure}

Our first goal is to construct an expressive model family that
combines all 3 paradigms while preserving their strengths.
The \textbf{Linear State-Space Layer (LSSL)} is a simple sequence model that maps a 1-dimensional function or sequence $u(t) \mapsto y(t)$ through an implicit state \( x(t) \) by simulating a linear continuous-time state-space representation in discrete-time
\begin{align}
  \dot{x}(t) &= Ax(t) + Bu(t) \label{eq:1} \\
  y(t) &= Cx(t) + Du(t) \label{eq:2}
  ,
\end{align}
where \( A \) controls the evolution of the system and \( B, C, D \) are projection parameters.
The LSSL can be viewed as an instantiation of each family, inheriting their strengths (\cref{fig:splash}):

\begin{itemize}
  \item \textbf{LSSLs are recurrent.}
    If a discrete step-size $\dt$ is specified, the LSSL can be discretized into a linear recurrence using standard techniques, and simulated during inference as a stateful recurrent model with constant memory and computation per time step.
  \item \textbf{LSSLs are convolutional.}
    The linear time-invariant systems defined by \eqref{eq:1}+\eqref{eq:2} are known to be explicitly representable as a continuous convolution.
    Moreover, the discrete-time version can be parallelized during training using convolutions~\citep{chilkuri2021parallelizing,romero2021ckconv}.
  \item \textbf{LSSLs are continuous-time.}
    The LSSL itself is a differential equation. As such, it can perform unique applications of continuous-time models, such as simulating continuous processes, handling missing data~\cite{rubanova2019latent}, and adapting to different timescales.
\end{itemize}

Surprisingly, we show that LSSLs do not sacrifice expressivity, and in fact generalize convolutions and RNNs.
First, classical results from control theory imply that all 1-D convolutional kernels can be approximated by an LSSL~\citep{williams2007linear}.
Additionally, we provide two results relating RNNs and ODEs that may be of broader interest,
e.g.\ showing that some RNN architectural heuristics (such as gating mechanisms) are related to the step-size \( \dt \) and can actually be derived from ODE approximations.
As corollaries of these results, we show that popular RNN methods are special cases of LSSLs.

The generality of LSSLs does come with tradeoffs.
In particular, we describe and address two challenges that naive LSSL instantiations face when handling long sequences:
(i) they inherit the limitations of both RNNs and CNNs at remembering long dependencies, and
(ii) choosing the state matrix \( A \) and timescale \( \dt \) appropriately are critical to their performance, yet learning them is computationally infeasible.
We simultaneously address these challenges by specializing LSSLs using a carefully chosen class of structured matrices \( A \),
such that
(i) these matrices generalize prior work on continuous-time memory \citep{gu2020hippo} and mathematically capture long dependencies with respect to a learnable family of measures, and
(ii) with new algorithms, LSSLs with these matrices \( A \) can be theoretically sped up under certain computation models, even while learning the measure \( A \) and timescale \( \dt \).

We empirically validate that LSSLs are widely effective on benchmark datasets and very long time series from healthcare sensor data, images, and speech.
\begin{itemize}[leftmargin=*]%
  \item
On benchmark datasets, LSSLs obtain
SoTA over recent RNN, CNN, and NDE-based methods
across sequential image classification tasks (e.g., by over 10\% accuracy on sequential CIFAR)
and healthcare regression tasks with length-4000 time series (by up to 80\% reduction in RMSE).

  \item
To showcase the potential of LSSLs to unlock applications with extremely long sequences,
we introduce a new sequential CelebA classification task with length-38000 sequences.
A small LSSL
comes within 2.16 accuracy points of a specialized ResNet-18 vision architecture that has 10x more parameters and is trained directly on images.

  \item
Finally, we test LSSLs on a difficult dataset of high-resolution speech clips,
where usual speech pipelines pre-process the signals to reduce the length by 100x.
When training on the \emph{raw} length-16000 signals,
the LSSL not only (i) outperforms previous methods %
by over 20 accuracy points in 1/5 the training time,
but (ii) outperforms all baselines that use the pre-processed length-160 sequences,
overcoming the limitations of hand-crafted feature engineering.
\end{itemize}

\textbf{Summary of Contributions}
\begin{itemize}[leftmargin=*]%
  \item We introduce Linear State-Space Layers (LSSLs), a simple sequence-to-sequence transformation that shares the modeling advantages of recurrent, convolutional, and continuous-time methods.
    Conversely, we show that RNNs and CNNs can be seen as special cases of LSSLs (\cref{sec:lssl}).
  \item We prove that a structured subclass of LSSLs can learn representations that solve continuous-time memorization, allowing it to adapt its measure and timescale (\cref{sec:mlssl}).
    We also provide new algorithms for these LSSLs, showing that they can be sped up computationally under an arithmetic complexity model \cref{sec:sllssl}.
  \item Empirically, we show that LSSLs stacked into a deep neural network are widely effective on time series data, even (or \emph{especially}) on extremely long sequences (\cref{sec:experiments}).
\end{itemize}

\section{Technical Background}
\label{sec:background}
We summarize the preliminaries on differential equations that are necessary for this work.
We first
introduce two standard approximation schemes for differential equations that we will use to convert continuous-time models to discrete-time, and will be used in our results on understanding RNNs.
We give further context on the step size or timescale \( \dt \), which is a particularly important parameter involved in this approximation process.
Finally, we provide a summary of the HiPPO framework for continuous-time memorization~\citep{gu2020hippo},
which will give us a mathematical tool for constructing LSSLs that can address long-term dependencies.

\paragraph{Approximations of differential equations.}
Any differential equation \( \dot{x}(t) = f(t, x(t)) \) has an equivalent \emph{integral equation} 
$ x(t) = x(t_0) + \int_{t_0}^{t} f(s, x(s)) \dd s$.
This can be numerically solved by storing some approximation for $x$,
and keeping it fixed inside $f(t, x)$ while iterating the equation.
For example, \emph{Picard iteration} is often used to prove the existence of solutions to ODEs by iterating the equation
$  x_{i+1}(t) := x_i(t_0) + \int_{t_0}^{t} f(s, x_i(s)) \dd s$
.
In other words, it finds a sequence of functions $x_0(t), x_1(t), \dots$ that approximate the solution \( x(t) \) of the integral equation.

\paragraph{Discretization.}
On the other hand,
for a desired sequence of discrete times \( t_i \), approximations to \( x(t_0), x(t_1), \dots \) can be found
by iterating the equation
$  x(t_{i+1}) = x(t_i) + \int_{t_i}^{t_{i+1}} f(s, x(s)) \dd s$.
Different ways of approximating the RHS integral lead to different discretization schemes.
We single out a discretization method called the \textbf{generalized bilinear transform (GBT)} which is specialized to linear ODEs of the form \eqref{eq:1}.
Given a \emph{step size} \( \dt \), the GBT update is
\begin{equation}
  \label{eq:gbt}
  x(t + \dt) = (I - \alpha \dt \cdot A)^{-1}(I + (1 - \alpha) \dt \cdot A) x(t) + \dt (I - \alpha \dt \cdot A)^{-1} B \cdot u(t)
  .
\end{equation}
Three important cases are:
\( \alpha=0 \) becomes the classic \emph{Euler method} which is simply the first-order approximation \( x(t+\dt) = x(t) + \dt \cdot x'(t) \);
\( \alpha=1 \) is called the \emph{backward Euler} method;
and \( \alpha=\frac{1}{2} \) is called the \emph{bilinear} method, which preserves the stability of the system~\cite{zhang2007performance}.

In \cref{sec:lssl-expressivity} we will show that the backward Euler method and Picard iteration are actually related to RNNs.
On the other hand, the bilinear discretization will be our main method for computing accurate discrete-time approximations of our continuous-time models.
In particular, define \( \overline{A} \) and \( \overline{B} \) to be the matrices appearing in \eqref{eq:gbt} for \( \alpha = \frac{1}{2} \).
Then the discrete-time state-space model
is
\begin{align}
  x_{t} &= \overline{A} x_{t-1} + \overline{B} u_t \label{eq:1-discrete} \\
  y_t &= C x_t + D u_t \label{eq:2-discrete}
  .
\end{align}

\paragraph{\( \dt \) as a timescale.}
In most models, the length of dependencies they can capture is roughly proportional to \( \frac{1}{\dt} \).
Thus we also refer to the step size \( \dt \) as a \emph{timescale}.
This is an intrinsic part of converting a continuous-time ODE into a discrete-time recurrence,
and most ODE-based RNN models have it as an important and non-trainable hyperparameter~\cite{voelker2019legendre,gu2020hippo,rusch2021unicornn}.
On the other hand, in \cref{sec:lssl-expressivity} we show that the gating mechanism of classical RNNs is a version of learning \( \dt \).
Moreover when viewed as a CNN, the timescale \( \dt \) can be viewed as controlling the width of the convolution kernel (\cref{sec:lssl-expressivity}).
Ideally, all ODE-based sequence models would be able to automatically learn the proper timescales.

\paragraph{Continuous-time memory.}
Consider an input function \( u(t) \), a fixed probability measure $\omega(t)$, and a sequence of \( N \) basis functions such as polynomials.
At every time \( t \), the history of \( u \) before time \( t \) can be projected onto this basis,
which yields a vector of coefficients \( x(t) \in \mathbb{R}^N \) that represents an optimal approximation of the history of $u$ with respect to the provided measure $\omega$.
The map taking the function \( u(t) \in \mathbb{R} \) to coefficients \( x(t) \in \mathbb{R}^N \) is called the \textbf{High-Order Polynomial Projection Operator (HiPPO)} with respect to the measure \( \omega \).
In special cases such as the uniform measure \( \omega = \mathbb{I}\{[0,1]\} \) and the exponentially-decaying measure \( \omega(t) = \exp(-t) \),
\citet{gu2020hippo} showed that \( x(t) \) satisfies a differential equation $\dot{x}(t) = A(t) x(t) + B(t) u(t)$ (i.e., \eqref{eq:1}) and derived closed forms for the matrix \( A \).
Their framework provides a principled way to design memory models handling long dependencies;
however, they prove only these few special cases.

\section{Linear State-Space Layers (LSSL)}
\label{sec:lssl}

We define our main abstraction, a model family that generalizes recurrence and convolutions.
\cref{sec:lssl-interpretation} first formally defines the LSSL, then discusses how to compute it
with multiple views.
Conversely, \cref{sec:lssl-expressivity} shows that LSSLs are related to mechanisms of the most popular RNNs.

\subsection{Different Views of the LSSL}
\label{sec:lssl-interpretation}

Given a fixed state space representation $A, B, C, D$, an LSSL is the sequence-to-sequence mapping defined by discretizing the linear state-space model (1) and (2).

Concretely, an LSSL layer has parameters
$A, B, C, D$, and $\dt$.
It operates on an input \( u \in \mathbb{R}^{L \times H} \) representing a sequence of length \( L \) where each timestep has an \( H \)-dimensional feature vector.
Each feature \( h \in [H] \) defines a sequence \( (u_t^{(h)})_{t \in [L]} \),
which is combined with a timescale \( \dt_h \) to define an output \( y^{(h)} \in \mathbb{R}^L \) via the discretized state-space model \eqref{eq:1-discrete}+\eqref{eq:2-discrete}.

Computationally, the discrete-time LSSL can be viewed in multiple ways (\cref{fig:splash}).

\paragraph{As a recurrence.}
The recurrent state
\( x_{t-1} \in \mathbb{R}^{H \times N} \)
carries the context of all inputs before time \( t \).
The current state \( x_t \) and output \( y_t \) can be computed by simply following equations \eqref{eq:1-discrete}+\eqref{eq:2-discrete}. %
Thus the LSSL is a recurrent model
with efficient and stateful inference,
which can consume a (potentially unbounded) sequence of inputs while requiring fixed computation/storage per time step.

\paragraph{As a convolution.}
For simplicity let the initial state be \( x_{-1} = 0 \). %
Then \eqref{eq:1-discrete}+\eqref{eq:2-discrete} explicitly yields %
\small
\begin{equation}
  \label{eq:convolution}
  y_k = C \left( \overline{A} \right)^k \overline{B} u_0 + C \left( \overline{A} \right)^{k-1} \overline{B} u_1 + \dots + C \overline{A} \overline{B} u_{k-1} + \overline{B} u_k
  + D u_k
  .
\end{equation}
\normalsize
Then \( y \) is simply the (non-circular) convolution \( y = \mathcal{K}_L(\overline{A}, \overline{B}, C) \ast u + D u \),
where
\begin{equation}%
  \label{eq:krylov}
  \mathcal{K}_L(A, B, C) = \left(C A^i B\right)_{i \in [L]} \in \mathbb{R}^L = (CB, CAB, \dots, CA^{L-1}B)
  .
\end{equation}

Thus the LSSL can be viewed as a convolutional model where
the entire output \( y \in \mathbb{R}^{H \times L} \) can be computed at once by a convolution,
which can be efficiently implemented with three FFTs.

\paragraph{The computational bottleneck.}
We make a note that the bottleneck of (i) the recurrence view is \textbf{matrix-vector multiplication (MVM)} by the discretized state matrix \( \overline{A} \) when simulating \eqref{eq:1-discrete},
and (ii) the convolutional view is computing the \textbf{Krylov function} $\mathcal{K}_L$ \eqref{eq:krylov}.
Throughout this section we assumed the LSSL parameters were fixed, which means that \( \overline{A} \) and \( \mathcal{K}_L(A, B, C) \) can be cached for efficiency.
However, learning the parameters \( \overline{A} \) and \( \dt \) would involve repeatedly re-computing these, which is infeasible in practice.
We revisit and solve this problem in \cref{sec:sllssl}.

\subsection{Expressivity of LSSLs}
\label{sec:lssl-expressivity}
For a model to be both recurrent and convolutional, one might expect it to be limited in other ways.
Indeed, while \citep{chilkuri2021parallelizing,romero2021ckconv} also observe that certain recurrences can be replaced with a convolution, they note that it is not obvious if convolutions can be replaced by recurrences.
Moreover, while the LSSL is a linear recurrence, popular RNN models are \emph{nonlinear} sequence models with activation functions between each time step.
We now show that LSSLs surprisingly do not have limited expressivity.

\paragraph{Convolutions are LSSLs.}
A well-known fact about state-space systems \eqref{eq:1}+\eqref{eq:2} is that the output $y$ is related to the input $u$ by a convolution $y(t) = \int h(\tau) u(t-\tau) d \tau$ with the \emph{impulse response} $h$ of the system.
Conversely, a convolutional filter $h$ that is a rational function of degree $N$ can be represented by a state-space model of size \( N \)~\citep{williams2007linear}.
Thus, an arbitrary convolutional filter \( h \) can be approximated by a rational function (e.g., by Pad\'e approximants) and represented by an LSSL.

In the particular case of LSSLs with HiPPO matrices (\cref{sec:background,sec:mlssl}), there is another intuitive interpretation of how LSSL relate to convolutions.
Consider the special case when \( A \) corresponds to a uniform measure (in the literature known as the LMU \citep{voelker2019legendre} or HiPPO-LegT \citep{gu2020hippo} matrix).
Then for a fixed \( dt \), equation \eqref{eq:1} is simply memorizing the input within sliding windows of \( \frac{1}{\dt} \) elements,
and equation \eqref{eq:2} extracts features from this window.
Thus the LSSL can be interpreted as automatically learning convolution filters with a learnable kernel width.

\paragraph{RNNs are LSSLs.}
We show two results about RNNs that may be of broader interest.
Our first result says that the ubiquitous \emph{gating mechanism} of RNNs, commonly perceived as a heuristic to smooth optimization \citep{lstm}, is actually the analog of a step size or timescale \( \dt \).
\begin{lemma}%
  \label{lmm:gates}
  A (1-D) gated recurrence \( x_{t} = (1-\sigma(z)) x_{t-1} + \sigma(z) u_t \), where \( \sigma \) is the sigmoid function and \( z \) is an arbitrary expression,
  can be viewed as the GBT(\( \alpha=1 \)) (i.e., \emph{backwards-Euler}) discretization of a 1-D linear ODE \( \dot{x}(t) = -x(t) + u(t) \).
\end{lemma}
\begin{proof}%
  Applying a discretization requires a positive step size \( \dt \).
  The simplest way to parameterize a positive function is via the exponential function \( \dt = \exp(z) \) applied to any expression \( z \).
  Substituting this into \eqref{eq:gbt} with \( A = -1, B = 1, \alpha = 1 \)
  exactly produces the gated recurrence.
\end{proof}

While~\cref{lmm:gates} involves approximating continuous systems using discretization, the second result is about approximating them using Picard iteration (\cref{sec:background}).
Roughly speaking, each layer of a deep \emph{linear} RNN can be viewed as successive Picard iterates \( x_0(t), x_1(t), ... \) approximating a function \( x(t) \) defined by a \emph{non-linear} ODE.
This shows that we do not lose modeling power by using linear instead of non-linear recurrences, and that the nonlinearity can instead be ``moved'' to the depth direction of deep neural networks to improve speed without sacrificing expressivity.
\begin{lemma}%
  \label{lmm:picard}
  (Infinitely) deep stacked LSSL layers of order $N=1$ with \emph{position-wise} non-linear functions can approximate any non-linear ODE $\dot{x}(t) = -x + f(t, x(t))$.
\end{lemma}
We note that many of the most popular and effective RNN variants
such as the LSTM~\citep{lstm}, GRU~\citep{chung2014empirical}, QRNN~\citep{bradbury2016quasi}, and SRU~\citep{lei2017simple},
involve a hidden state \( x_t \in \mathbb{R}^H \) that involves independently ``gating'' the \( H \) hidden units.
Applying \cref{lmm:gates}, they actually also approximate an ODE of the form in \cref{lmm:picard}.
Thus LSSLs and these popular RNN models can be seen to all approximate the same type of underlying continuous dynamics,
by using Picard approximations in the depth direction and discretization (gates) in the time direction.
\cref{sec:lssl-proofs} gives precise statements and proofs.

\subsection{Deep LSSLs}

The basic LSSL is defined as a sequence-to-sequence map from \( \mathbbm{R}^L \to \mathbbm{R}^L \) on 1D sequences of length \( L \),
parameterized by parameters \( A \in \mathbbm{R}^{N \times N}, B \in \mathbbm{R}^{N \times 1}, C \in \mathbbm{R}^{1 \times N}, D \in \mathbbm{R}^{1 \times 1}, \dt \in \mathbbm{R} \).
Given an input sequence with hidden dimension \( H \) (in other words a feature dimension greater than \( 1 \)),
we simply broadcast the parameters \( B, C, D, \dt \) with an extra dimension \( H \).
Each of these \( H \) copies is learned independently, so that there are \( H \) different versions of a 1D LSSL processing each of the input features independently.
Overall, the standalone LSSL layer is a sequence-to-sequence map with the same interface as standard sequence model layers such as RNNs, CNNs, and Transformers.

The full LSSL architecture in a deep neural network is defined similarly to standard sequence models such as deep ResNets and Transformers,
involving stacking LSSL layers connected with normalization layers and residual connections.
Full architecture details are described in \cref{sec:model}, including the initialization of \( A \) and \( \dt \),
computational details,
and other architectural details.

\section{Combining LSSLs with Continuous-time Memorization}
\label{sec:structured-lssl}

In \cref{sec:lssl} we introduced the LSSL model and showed that it shares the strengths of convolutions and recurrences while also generalizing them.
We now discuss and address its main limitations,
in particular handling long dependencies (\cref{sec:mlssl}) and efficient computation (\cref{sec:sllssl}).

\subsection{Incorporating Long Dependencies into LSSLs}
\label{sec:mlssl}

The generality of LSSLs means they can inherit the issues of recurrences and convolutions at addressing long dependencies (\cref{sec:intro}). %
For example, viewed as a recurrence, repeated multiplication by \( \overline{A} \) could suffer from the \emph{vanishing gradients} problem \cite{pascanu2013difficulty,romero2021ckconv}.
We confirm empirically that LSSLs with random state matrices \( A \) are actually not effective (\cref{sec:experiments-ablations}) as a generic sequence model.

However, one advantage of these mathematical continuous-time models is that they are theoretically analyzable, and specific \( A \) matrices can be derived to address this issue.
In particular, %
the HiPPO framework (\cref{sec:background}) describes how to memorize a function in continuous time with respect to a measure \( \omega \)~\citep{gu2020hippo}.
This operator mapping a function to a continuous representation of its past is denoted \( \hippo(\omega) \), and was shown to have the form of equation~\eqref{eq:1} in three special cases.
However, these matrices are non-trainable in the sense that no other \( A \) matrices were known to be \( \hippo \) operators.

To address this, we theoretically
resolve the open question from \citep{gu2020hippo}, showing that \( \hippo(\omega) \) for \emph{any measure} \( \omega \)
\footnote{To be precise, the measures that correspond to orthogonal polynomials \citep{szego}.}
results in \eqref{eq:1} with a structured matrix \( A \).
\begin{theorem}[Informal]%
  \label{thm:hippo}
  For an arbitrary measure \( \omega \), the optimal memorization operator \( \operatorname*{hippo}(\omega) \) has the form \( \dot{x}(t) = A x(t) + B u(t) \) \eqref{eq:1} for a \emph{low recurrence-width (LRW)}~\citep{de2018two} state matrix \( A \).
\end{theorem}
For measures covering the classical orthogonal polynomials (OPs)~\cite{szego} (in particular, corresponding to Jacobi and Laguerre polynomials), there is even more structure.
\begin{corollary}%
  \label{thm:jacobi}
  For \( \omega \) corresponding to the classical OPs,
  \( \hippo(\omega) \) is \( 3 \)-quasiseparable.
\end{corollary}

Although beyond the scope of this section, we mention that LRW matrices are a type of structured matrix that have linear MVM~\citep{de2018two}.
In \cref{sec:sllssl-proofs} we define this class and prove \cref{thm:hippo}.
Quasi-separable matrices are a related class of structured matrices with additional algorithmic properties.
We define these matrices in \cref{def:quasi} and prove \cref{thm:jacobi} in \cref{sec:sllssl:jacobi}.

\cref{thm:hippo} tells us that a LSSL that uses a state matrix \( A \)
within a particular class of structured matrices
would carry the theoretical interpretation of continuous-time memorization.
Ideally, we would be able to automatically learn the best \( A \) within this class;
however, this runs into computational challenges which we address next (\cref{sec:sllssl}).
For now, we define the \textbf{LSSL-fixed} or \textbf{LSSL-f} to be one where the \( A \) matrix is fixed to one of the HiPPO matrices prescribed by \citep{gu2020hippo}.

\subsection{Theoretically Efficient Algorithms for the LSSL}
\label{sec:sllssl}
Although \( A \) and \( \dt \) are the most critical parameters of an LSSL which govern the state-space (c.f.\ \cref{sec:mlssl}) and timescale (\cref{sec:background,sec:lssl-expressivity}), they are not feasible to train in a naive LSSL.
In particular, \cref{sec:lssl-interpretation} noted that it would require efficient \emph{matrix-vector multiplication (MVM)} and \emph{Krylov function} \eqref{eq:krylov} for \( \overline{A} \) to compute the recurrent and convolutional views, respectively.
However, the former seems to involve a matrix inversion \eqref{eq:gbt},
while the latter seems to require powering \( \overline{A} \) up \( L \) times.

In this section, we show that the same restriction of \( A \) to the class of quasiseparable (\cref{thm:jacobi}), which gives an LSSL the ability to theoretically remember long dependencies,
simultaneously grants it computational efficiency.

First of all, it is known that quasiseparable matrices have efficient (linear-time) MVM~\citep{pernet2016computing}.
We show that they also have fast Krylov functions, allowing efficient training with convolutions.
\begin{theorem}%
  \label{thm:krylov}
  For any \( k \)-quasiseparable matrix \( A \) (with constant \( k \))
  and arbitrary \( B, C \), the Krylov function \( \mathcal{K}_L(A, B, C) \) can be computed in \emph{quasi-linear} time and space \( \tilde{O}(N + L) \) and \emph{logarithmic} depth (i.e., is parallelizable).
  The operation count is in an exact arithmetic model, not accounting for bit complexity or numerical stability.
\end{theorem}
We remark that \cref{thm:krylov} is non-obvious.
To illustrate, it is easy to see that unrolling \eqref{eq:krylov} for a general matrix \( A \) takes time \( LN^2 \).
Even if \( A \) is extremely structured with linear computation, it requires \( LN \) operations and linear depth.
The depth can be reduced with the squaring technique (batch multiply by \( A, A^2, A^4, \dots \)),
but this then requires \( LN \) intermediate storage.
In fact, the algorithm for \cref{thm:krylov} is quite sophisticated (\cref{sec:sllssl-algorithms}) and
involves a divide-and-conquer recursion over matrices of \emph{polynomials}, using the observation that \eqref{eq:krylov} is related to the power series \( C(I - Ax)^{-1}B \) .

Unless specified otherwise, the full \textbf{LSSL} refers to an LSSL with \( A \)
satisfying \cref{thm:jacobi}.
In conclusion, learning within this structured matrix family simultaneously endows LSSLs with long-range memory through \cref{thm:hippo} and is theoretically computationally feasible through \cref{thm:krylov}.
We note the caveat that \cref{thm:krylov} is over exact arithmetic and not floating point numbers, and thus is treated more as a proof of concept that LSSLs can be computationally efficient in theory.
We comment more on the limitations of the LSSL in \cref{sec:discussion}.

\section{Empirical Evaluation}
\label{sec:experiments}

\begin{figure}[!t]
\begin{minipage}[t]{0.45\linewidth}
  \small
  \centering
  \captionsetup{type=table}
  \caption{
    (\textbf{Pixel-by-pixel image classification.})
    (Top) our methods.
    (Middle) recurrent baselines.
    (Bottom) convolutional + other baselines.
  }
    \begin{tabular}{@{}llll@{}}
      \toprule
      Model                                    & sMNIST           & pMNIST            & sCIFAR           \\
      \midrule
      \textbf{LSSL}                            & \textbf{99.53}   & \textbf{98.76}    & \textbf{84.65}   \\
      \textbf{LSSL-fixed}                      & \textbf{99.50}   & \textbf{98.60}    & \textbf{81.97}   \\
      \midrule
      LipschitzRNN                             & \underline{99.4} & 96.3              & 64.2             \\
      LMUFFT~\citep{chilkuri2021parallelizing} & -                & 98.49             & -                \\
      UNIcoRNN~\citep{rusch2021unicornn}       & -                & 98.4              & -                \\
      HiPPO-RNN~\citep{gu2020hippo}            & 98.9             & 98.3              & 61.1             \\
      URGRU~\citep{gu2020improving}            & 99.27            & 96.51             & \underline{74.4} \\
      IndRNN~\citep{indrnn}                    & 99.0             & 96.0              & -                \\
      Dilated RNN~\citep{chang2017dilated}     & 98.0             & 96.1              & -                \\
      r-LSTM ~\citep{trinh2018learning}        & 98.4             & 95.2              & 72.2             \\
      \midrule
      CKConv~\citep{romero2021ckconv}          & 99.32            & \underline{98.54} & 63.74            \\
      TrellisNet~\citep{trellisnet}            & 99.20            & 98.13             & 73.42            \\
      TCN~\citep{bai2018empirical}             & 99.0             & 97.2              & -                \\
      Transformer~\citep{trinh2018learning}    & 98.9             & 97.9              & 62.2             \\
      \bottomrule
    \end{tabular}
    \label{tab:image}
\end{minipage}
\hfill
\begin{minipage}[t]{0.45\linewidth}
  \small
  \centering
  \captionsetup{type=table}
  \caption{
    (\textbf{Vital signs prediction.})
    RMSE for predicting
    respiratory rate (RR), heart rate (HR), and blood oxygen (SpO2).
    * indicates our own runs
    to complete results for the strongest baselines.
  }
    \begin{tabular}{@{}llll@{}}
      \toprule
      Model                              & RR               & HR               & SpO2               \\
      \midrule
      \textbf{LSSL}                     & \textbf{0.350}   & \textbf{0.432}   & \textbf{0.141}     \\
      \textbf{LSSL-fixed}                      & \textbf{0.378} & \textbf{0.561}   &     \textbf{0.221} \\
      \midrule
      UnICORNN~\citep{rusch2021unicornn} & \underline{1.06} & \underline{1.39} & \underline{0.869}* \\
      coRNN~\citep{rusch2021unicornn}    & 1.45             & 1.81             & -                  \\
      CKConv                             & 1.214*           & 2.05*            & 1.051*             \\
      NRDE~\citep{morrill2021neural}     & 1.49             & 2.97             & 1.29               \\
      IndRNN~\citep{rusch2021unicornn}   & 1.47             & 2.1              & -                  \\
      expRNN~\citep{rusch2021unicornn}   & 1.57             & 1.87             & -                  \\
      LSTM                               & 2.28             & 10.7             & -                  \\
      Transformer                        & 2.61*            & 12.2*            & 3.02*              \\

      \midrule
      XGBoost~\citep{Tan2020TSER}        & 1.67             & 4.72             & 1.52               \\
      Random Forest~\citep{Tan2020TSER}  & 1.85             & 5.69             & 1.74               \\
      Ridge Regress.~\citep{Tan2020TSER} & 3.86             & 17.3             & 4.16               \\
      \bottomrule
    \end{tabular}
    \label{tab:bidmc}

\end{minipage}
\end{figure}

We test LSSLs empirically on a range of time series datasets with sequences from length 160 up to 38000 (\cref{sec:experiments-benchmarks,sec:experiments-speech}),
where they substantially improve over prior work.
We additionally validate the computational and modeling benefits of LSSLs from generalizing all three main model families (\cref{sec:experiments-expressivity}),
and analyze the benefits of incorporating principled memory representations that can be learned (\cref{sec:experiments-ablations}).

\paragraph{Baselines.}
Our tasks have extensive prior work and we evaluate against previously reported best results.
We highlight our primary baselines, three very recent works explicitly designed for long sequences:
CKConv (a continuous-time CNN)~\citep{romero2021ckconv}, UnICORNN (an ODE-inspired RNN)~\citep{rusch2021unicornn},
and Neural Controlled/Rough Differential Equations (NCDE/NRDE) (a sophisticated NDE)~\citep{kidger2020neural,morrill2021neural}. %
These are the only models we are aware of that have experimented with sequences of length >10k.

\subsection{Image and Time Series Benchmarks}
\label{sec:experiments-benchmarks}

\begin{wraptable}{r}{0.25\linewidth}%
  \vspace*{-1.7em}
\small
 \centering
\caption{
    (\textbf{Sequential CelebA Classification}.)
  }
\resizebox{\linewidth}{!}{%
\begin{tabular}{@{}lll@{}}
  \toprule
               & \textbf{LSSL-f} & \textbf{ResNet} \\
                        \midrule
\textbf{Att.}  & 78.89           & 81.35           \\
\textbf{MSO}   & 92.36           & 93.92           \\
\textbf{Smil.} & 90.95           & 92.89           \\
\textbf{WL}    & 90.57           & 93.25           \\
\bottomrule
\end{tabular}%
}

\label{tab:celeba}
\end{wraptable}

We test on the sequential MNIST, permuted MNIST, and sequential CIFAR tasks (\cref{tab:image}),
popular benchmarks which were originally designed to test the ability of recurrent models to capture long-term dependencies of length up to 1k \citep{arjovsky2016unitary}.
LSSL sets SoTA on sCIFAR by more than 10 points.
We note that all results were achieved with at least 5x fewer parameters than the previous SoTA (\cref{sec:experiments-full}).

We additionally use the BDIMC healthcare datasets (\cref{tab:bidmc}), a suite of widely studied time series regression problems of length 4000 on estimating vital signs. 
LSSL reduces RMSE by more than two-thirds on all datasets.

\subsection{Speech and Image Classification for \texorpdfstring{\emph{Very}}{Very} Long Time Series}
\label{sec:experiments-speech}

Raw speech is challenging for ML models due to high-frequency sampling resulting in very long sequences. Traditional systems involve complex pipelines that require
feeding mixed-and-matched hand-crafted features into DNNs~\cite{pytorch-kaldi}.
\cref{tab:sc} reports results for the Speech Commands (SC) dataset~\citep{kidger2020neural} for classification of 1-second audio clips.
Few methods have made progress on the raw speech signal,
instead requiring pre-processing with standard mel-frequency cepstrum coefficients (MFCC).
By contrast, LSSL sets SoTA on this dataset \emph{while training on the raw signal}.
We note that MFCC extracts sliding window frequency coefficients and thus is related to the coefficients \( x(t) \) defined by LSSL-f (\cref{sec:background}, \cref{sec:mlssl}, \cite{gu2020hippo}, \cref{sec:sllssl-proofs}).
Consequently, LSSL may be interpreted as automatically learning MFCC-type features in a trainable basis.

To stress-test the LSSL's ability to handle extremely long sequences, we create a challenging new sequential-CelebA task, where we classify $178 \times 218$ images = \textbf{38000-length} sequences for 4 facial attributes: Attractive (Att.), Mouth Slightly Open (MSO), Smiling (Smil.), Wearing Lipstick (WL) \cite{liu2015image}. We chose the 4 most class-balanced attributes to avoid well-known problems with class imbalance. LSSL-f comes close to matching the performance of a specialized ResNet-18 image classification architecture that has $10\times$ the parameters (Table~\ref{tab:celeba}). We emphasize we are the first to demonstrate that this is possible to do with a generic sequence model.

\begin{table}[t]
  \caption{
    (\textbf{Raw Speech Classification; Timescale Shift}.)
    (Top): Raw signals (length 16000);
    \( 1 \to f \) indicates test-time change in sampling rate by a factor of \( f \). %
    (Bottom): Pre-processed MFCC features used in prior work (length 161). %
    \xmark{} denotes computationally infeasible.
  }
  \small
  \centering
  \begin{tabular}{@{}llllllll@{}}
    \toprule
                            & \textbf{LSSL} & \textbf{LSSL-f} & CKConv        & UnICORNN & N(C/R)DE & ODE-RNN~\citep{rubanova2019latent} & GRU-ODE~\citep{de2019gru} \\
    \midrule
    \( 1 \to 1 \)           & \textbf{95.87}  & 90.64          & 71.66         & 11.02    & 16.49    & \xmark                             & \xmark                    \\
    \( 1 \to \frac{1}{2} \) & \textbf{88.66}  & 78.01          & 65.96         & 11.07    & 15.12        & \xmark                             & \xmark                    \\
    \midrule
    MFCC                    & 93.58         & 92.55        & \textbf{95.3} & 90.64    & 89.8     & 65.9                               & 47.9                      \\
    \bottomrule
  \end{tabular}
  \label{tab:sc}
\end{table}

\subsection{Advantages of Recurrent, Convolutional, and Continuous-time Models}
\label{sec:experiments-expressivity}

We validate that the generality of LSSLs endows it with the strengths of all three families.

\textbf{Convergence Speed.}
As a recurrent and NDE model that incorporates new theory for continuous-time memory (\cref{sec:mlssl}), the LSSL has strong inductive bias for sequential data,
and converges rapidly to SoTA results on our benchmarks.
With its convolutional view, training can be parallelized and it is also computationally efficient in practice.
\cref{tab:speed-training} compares the time it takes the LSSL-f to achieve SoTA, in either sample (measured by epochs) or computational (measured by wall clock) complexity.
In all cases, LSSLs reached the target in a fraction of the time of the previous model.

\begin{table}
  \small
  \caption{
    (\textbf{Modeling and Computational Benefits of LSSLs.})
    In each benchmark category, we compare the number of epochs (ep.) it takes a LSSL-f to reach the previous SoTA (PSoTA) results as well as a near-SoTA target.
    We also report the wall clock time it took to reach PSoTA relative to the previous best model. %
  }
    \centering
    \resizebox{\linewidth}{!}{%
    \begin{tabular}{@{}llllllllll@{}}
        \toprule
                       & \multicolumn{3}{c}{Permuted MNIST} & \multicolumn{3}{c}{BDIMC Heart Rate} & \multicolumn{3}{c}{Speech Commands RAW} \\
                                            \cmidrule(lr){2-4} \cmidrule(lr){5-7} \cmidrule(lr){8-10}
                       & 98\% Acc.                          & PSoTA                                & Time                                     & 1.5 RMSE & PSoTA   & Time             & 65\% Acc. & PSoTA   & Time             \\
        \midrule
        \textbf{LSSL-fixed} & 16 ep.                             & 104 ep.                              & 0.19\( \times \)                         & 9 ep.    & 10 ep.  & 0.07\( \times \) & 9 ep.     & 10 ep.  & 0.14\( \times \) \\
        CKConv         & 118 ep.                            & 200 ep.                              & 1.0\( \times \)                          & \xmark   & \xmark  & \xmark           & 188 ep.   & 280 ep. & 1.0\( \times \)  \\
        UnICORNN       & 75 ep.                             & \xmark                               & \xmark                                   & 116 ep.  & 467 ep. & 1.0\( \times \)  & \xmark    & \xmark  & \xmark           \\
        \bottomrule
    \end{tabular}%
    }
    \label{tab:speed-training}
\end{table}

\textbf{Timescale Adaptation.}
\cref{tab:sc} also reports the results of continuous-time models that are able to handle unique settings such as
missing data in time series,
or test-time shift in timescale (we note that this is a realistic problem, e.g., when deployed healthcare models are tested on EEG signals that are sampled at a different rate~\cite{saab2020weak,shah2018temple}).
We note that many of these baselines were custom designed for such settings, which is of independent interest.
On the other hand, LSSLs perform timescale adaptation by simply changing its \( \dt \) values at inference time, while still outperforming the performance of prior methods with no shift.
Additional results on the CharacterTrajectories dataset from prior work~\citep{kidger2020neural,romero2021ckconv} are in \cref{sec:experiments-full}, where LSSL is competitive with the best baselines.

\subsection{LSSL Ablations: Learning the Memory Dynamics and Timescale}
\label{sec:experiments-ablations}

We demonstrate that the \( \dt \) and \( A \) parameters, which LSSLs are able to automatically learn in contrast to prior work, are indeed critical to the performance of these continuous-time models.
We note that learning \( \dt \) adds only \( O(H) \) parameters and learning \( A \) adds \( O(N) \) parameters, adding less than 1\% parameter count compared to the base models with \( O(HN) \) parameters.

\textbf{Memory dynamics \( A \).}
We validate that vanilla LSSLs suffer from the modeling issues described in \cref{sec:structured-lssl}. %
We tested that LSSLs with \emph{random} \( A \) matrices (normalized appropriately) perform very poorly (e.g., 62\% on pMNIST).
Further, we note the consistent increase in performance from LSSL-f to LSSL despite the negligible parameter difference.
These ablations show that
(i) incorporating the theory of \cref{thm:hippo} is actually \emph{necessary} for LSSLs,
and (ii) further training the structured \( A \) is additionally helpful, which can be interpreted as learning the measure for memorization (\cref{sec:mlssl}).

\textbf{Timescale \( \dt \).}
\cref{sec:lssl-expressivity} showed that LSSL's ability to learn \( \dt \) is its direct generalization of the critical \emph{gating mechanism} of popular RNNs,
which previous ODE-based RNN models \citep{voelker2019legendre,gu2020hippo,chilkuri2021parallelizing,rusch2021unicornn} cannot learn.
We note that on sCIFAR, LSSL-f with poorly-specified $\dt$ gets only $49.3\%$ accuracy.
Additional results in \cref{sec:experiments-full} show that learning \( \dt \) alone provides an orthogonal boost to learning \( A \),
and visualizes the noticeable change in \( \dt \) over the course of training.

\section{Discussion} \label{sec:discussion}

In this work we introduced a simple and principled model (LSSL) inspired by a fundamental representation of physical systems.
We showed theoretically and empirically that it generalizes and inherits the strengths of the main families of modern time series models,
that its main limitations of long-term memory can be resolved with new theory on continuous-time memorization,
and that it is empirically effective on difficult tasks with very long sequences.

\textbf{Related work.}
The LSSL is related to several rich lines of work on recurrent, convolutional, and continuous-time models,
as well as sequence models addressing long dependencies.
\cref{sec:related} provides an extended related work connecting these topics.

\textbf{Tuning.}
Our models are very simple, consisting of identical L(\emph{inear})SSL layers with simple position-wise non-linear modules between layers (\cref{sec:model}).
Our models were able to train at much higher learning rates than baselines and were not sensitive to hyperparameters, of which we did light tuning primarily on learning rate and dropout.
In contrast to previous baselines~\citep{trellisnet,kidger2020neural,romero2021ckconv}, we did not use hyperparameters for improving stability and regularization such as weight decay, gradient clipping, weight norm, input dropout, etc.
While the most competitive recent works introduce at least one hyperparameter of critical importance (e.g. depth and step size~\citep{morrill2021neural}, \( \alpha \) and \( \dt \)~\citep{rusch2021unicornn}, \( \omega_0 \)~\citep{romero2021ckconv}) that are difficult to tune, the LSSL-fixed has only \( \dt \), which the full LSSL can even learn automatically (at the expense of speed).

\textbf{Limitations.}
\cref{sec:intro,sec:lssl,fig:splash} mention that a potential benefit of having the recurrent representation of LSSLs may endow it with efficient inference.
While this is theoretically possible, this work did not experiment on any applications that leverage this.
Follow-up work showed that it is indeed possible in practice to speed up some applications at inference time.

\cref{thm:krylov}'s algorithm is sophisticated (\cref{sec:sllssl-proofs}) and was not implemented in the first version of this work.
A follow-up to this paper found that it is not numerically stable and thus not usable on hardware.
Thus the algorithmic contributions in \cref{thm:krylov} serve the purpose of a proof-of-concept that fast algorithms for the LSSL do exist in other computation models (i.e., arithmetic operations instead of floating point operations), and leave an open question as to whether fast, numerically stable, and practical algorithms for the LSSL exist.

As described in \cref{sec:model},
by freezing the \( A \) matrix and \( \dt \)  timescale,
the LSSL-fixed is able to be computed much faster than the full LSSL,
and is comparable to prior models in practice (\cref{tab:speed-training}).
However, beyond computational complexity, there is also a consideration of space efficiency.
Both the LSSL and LSSL-fixed suffer from a large amount of space overhead (described in \cref{sec:model}) -- using \( O(NL) \) instead of \( O(L) \) space when working on a 1D sequence of length \( L \) -- that essentially stems from using the latent state representation of dimension \( N \).
Consequently, the LSSL can be space inefficient and we used multi-GPU training for our largest experiments (speech and high resolution images, \cref{tab:sc,tab:celeba}).

These fundamental issues with computation and space complexity were revisited and resolved in follow-up work to this paper, where a new state space model (the Structured State Space) provided a new parameterization and algorithms for state spaces.

\textbf{Conclusion and future work.}
Modern deep learning models struggle in applications with very long temporal data such as speech, videos, and medical time-series.
We hope that our conceptual and technical contributions can lead to new capabilities
with simple, principled, and less engineered models.
We note that our pixel-level image classification experiments, which use no heuristics (batch norm, auxiliary losses) or extra information (data augmentation),
perform similar to early convnet models with vastly more parameters,
and is in the spirit of recent attempts at unifying data modalities with a generic sequence model~\citep{dosovitskiy2020image}.
Our speech results demonstrate the possibility of \emph{learning} better features than hand-crafted processing pipelines used widely in speech applications.
We are excited about potential downstream applications, such as training other downstream models on top of pre-trained state space features.

\subsection*{Acknowledgments}
We thank Arjun Desai, Ananya Kumar, Laurel Orr, Sabri Eyuboglu, Dan Fu, Mayee Chen, Sarah Hooper, Simran Arora, and Trenton Chang for helpful feedback on earlier drafts.
We thank David Romero and James Morrill for discussions and additional results for baselines used in our experiments.
This work was done with the support of Google Cloud credits under HAI proposals 540994170283 and 578192719349.
AR and IJ are supported under NSF grant CCF-1763481.
KS is supported by the Wu Tsai Neuroscience Interdisciplinary Graduate Fellowship.
We gratefully acknowledge the support of NIH under No. U54EB020405 (Mobilize), NSF under Nos. CCF1763315 (Beyond Sparsity), CCF1563078 (Volume to Velocity), and 1937301 (RTML); ONR under No. N000141712266 (Unifying Weak Supervision); ONR N00014-20-1-2480: Understanding and Applying Non-Euclidean Geometry in Machine Learning; N000142012275 (NEPTUNE); the Moore Foundation, NXP, Xilinx, LETI-CEA, Intel, IBM, Microsoft, NEC, Toshiba, TSMC, ARM, Hitachi, BASF, Accenture, Ericsson, Qualcomm, Analog Devices, the Okawa Foundation, American Family Insurance, Google Cloud, Salesforce, Total, the HAI-AWS Cloud Credits for Research program, the Stanford Data Science Initiative (SDSI), and members of the Stanford DAWN project: Facebook, Google, and VMWare. The Mobilize Center is a Biomedical Technology Resource Center, funded by the NIH National Institute of Biomedical Imaging and Bioengineering through Grant P41EB027060. The U.S. Government is authorized to reproduce and distribute reprints for Governmental purposes notwithstanding any copyright notation thereon. Any opinions, findings, and conclusions or recommendations expressed in this material are those of the authors and do not necessarily reflect the views, policies, or endorsements, either expressed or implied, of NIH, ONR, or the U.S. Government.

\bibliography{biblio}
\bibliographystyle{plainnat}

\newpage

\appendix

\section{Related Work}
\label{sec:related}

We provide an extended related work comparing the LSSL to previous recurrent, convolutional, and continuous-time models.

\paragraph{HiPPO}

The LSSL is most closely related to the HiPPO framework for continuous-time memory \citep{gu2020hippo} and its predecessor, the Legendre Memory Unit (LMU) \citep{voelker2019legendre}.
The HiPPO-RNN and the LMU define dynamics of the form of equation \eqref{eq:1}, and incorporate it into an RNN architecture.
A successor to the LMU, the LMU-FFT \citep{chilkuri2021parallelizing} keeps the original linear dynamics, allowing the LMU to be computed with a cached convolution kernel.

These methods all suffer from two main limitations.
First, the state matrix \( A \) and discretization timescale \( \dt \) cannot be trained due to both limitations in theoretical understanding of which \( A \) matrices are effective, as well as computational limitations.
Second, \eqref{eq:1} is a 1-D to \( N \)-D map, requiring states to be projected back down to 1-D.
This creates an overall 1-D bottleneck in the state, limiting the expressivity of the model.

Compared to these, the LSSL does not use a conventional RNN architecture, instead keeping the linear recurrence \eqref{eq:1-discrete} and downprojecting it with the second part of the state space representation \eqref{eq:2-discrete}.
To avoid the 1-D feature bottlneck, it simply computes \( H \) copies of this 1-D to 1-D independently, creating an overall \( H \)-dimensional sequence-to-sequence model.
However, this exacerbates the computational issue, since the work is increased by a factor of \( H \).

This work resolves the expressivity issue with new theory.
Compared to HiPPO and the LMU, LSSL allows training the \( A \) matrix by showing generalized theoretical results for the HiPPO framework,
showing that there is a parameterized class of structured state spaces that are HiPPO operators.

The LSSL makes progress towards the second issue with new algorithms for these structured matrices (\cref{thm:krylov}).
However, as noted in \cref{sec:sllssl,sec:discussion}, the algorithm presented in \cref{thm:krylov} was later found to be not practical,
and an improved representation and algorithm was found in subsequent work.

\paragraph{Continuous-time CNNs.}
The CKConv is the only example of a continuous-time CNN that we are aware of, and is perhaps the strongest baseline in our experiments.
Rather than storing a finite sequence of weights for a convolution kernel,
the CKConv parameterizes it as an implicit function from $[0,1] \to \mathbb{R}$ which allows sampling it at any resolution.
A successor to the CKConv is the FlexConv~\citep{romero2021flexconv},
which learns convolutional kernels with a flexible width.
This is similar to the convolution interpretation of LSSL when using certain HiPPO bases (\cref{sec:lssl-expressivity}).

\paragraph{Continuous-time RNNs.}
The connection from RNNs to continuous-time models have been known since their inception, and recent years have seen an explosion of CT-RNN (continuous-time RNN) models based on dynamical systems or ODEs. We briefly mention a few classic and modern works along these lines, categorizing them into a few main topics.

First are theoretical works that analyze the expressivity of RNNs from a continuous-time perspective.
The connection between RNNs and dynamical systems has been studied since the 90s \citep{funahashi1993approximation}, fleshing out the correspondence between different dynamical systems and RNN architectures \citep{niu2019recurrent}.
Modern treatments have focused on analyzing the stability~\citep{zhang2014comprehensive} and dynamics~\citep{jordan2021gated} of RNNs.

Second, a large class of modern RNNs have been designed that aim to combat vanishing gradients from a dynamical systems analysis.
These include include the AntisymmetricRNN~\citep{chang2019antisymmetricrnn}, iRNN~\citep{kag2020rnns}, and LipschitzRNN~\citep{erichson2021lipschitz}, which address the exploding/vanishing gradient problem by reparatermizing the architecture or recurrent matrix based on insights from an underlying dynamical system.

Third is a class of models that are based on an explicit underlying ODE introduced to satisfy various properties.
This category includes the UnICORNN~\citep{rusch2021unicornn} and its predecessor coRNN~\citep{rusch2021coupled} which discretize a second-order ODE inspired by oscillatory systems.
Other models include the Liquid Time-Constant Networks (LTC)~\citep{hasani2021liquid} and successor CfC~\citep{hasani2021closed}, which use underlying dynamical systems with varying time-constants with stable behavior and provable rates of expressivity measured by trajectory length.
The LTC is based on earlier dynamic causal models (DCM)~\citep{friston2003dynamic}, which are a particular ODE related to state spaces with an extra bilinear term.
Finally, the LMU~\citep{voelker2019legendre} and HiPPO~\citep{gu2020hippo} also fall in this category, whose underlying ODEs are mathematically derived for continuous-time memorization.

Fourth, the recent family of neural ODEs~\citep{chen2018neural}, originally introduced as continuous-depth models, have been adapted to continuous-time, spawning a series of ``ODE-RNN'' models.
Examples include the ODE-RNN~\citep{rubanova2019latent}, GRU-ODE-Bayes~\citep{de2019gru}, and ODE-LSTM~\citep{lechner2020learning}, which extend adjoint-based neural ODEs to the discrete input setting as an alternative to standard RNNs.
Neural Controlled Differential Equations (NCDE)~\citep{kidger2020neural} and Neural Rough Differential Equations (NRDE)~\citep{morrill2021neural} are memory efficient versions that integrate observations more smoothly and can be extended to very long time series.

\paragraph{Gating mechanisms.}
As a special case of continuous-time RNNs, some works have observed the relation between gating mechanisms and damped dynamical systems~\citep{tallec2018can}.
Some examples of continuous-time RNNs based on such damped dynamical systems include the LTC~\citep{hasani2021liquid} and iRNN~\citep{kag2020rnns}.
Compared to these, \cref{lmm:gates} shows a stronger result that sigmoid gates are not just motivated by being an arbitrary monotonic function with range \( (0, 1) \), but the \emph{exact formula} appears out of discretizing a damped ODE.

\section{Model Details}
\label{sec:model}

\subsection{(M)LSSL Computation}

\cref{sec:lssl-interpretation} noted that some of the computations for using the LSSL are expensive to compute.
When the LSSL fixes the \( A \) and \( \dt \) parameters (e.g. when they are not trained, or at inference time),
these computational difficulties can be circumvented by caching particular computations.
In particular, this case applies to the LSSL-f.
Note that in this case, the other state-space matrices \( C \) and \( D \) comprise the \( O(HN) \) trainable parameters of the fixed-transition LSSL.

In particular, we assume that there is a black-box inference algorithm for this system, i.e. matrix-vector multiplication by \( \overline{A} \) (an example of implementing this black box for a particular structured class is in \cref{sec:sllssl:trid}).
We then compute and cache
\begin{itemize}%
  \item the transition matrix \( \overline{A} \), which is computed by applying the black-box \( \overline{A} \) MVM algorithm to the identity matrix \( I \).
  \item the \emph{Krylov matrix}
    \begin{equation}
      \label{eq:krylov-matrix}
      K(\overline{A}, \overline{B}) = (\overline{B}, \overline{A} \overline{B}, (\overline{A})^2 \overline{B}, ...) \in \mathbb{R}^{N \times L}
      ,
    \end{equation}
    which is computed in a parallelized manner by the squaring technique for exponentiation, i.e.\ batch multiply by \( \overline{A}, (\overline{A})^2, (\overline{A})^4, \dots \).
\end{itemize}

At inference time, the model can be unrolled recurrently with \( \overline{A} \).
At training time, the convolutional filter \( K_L(\overline{A}, \overline{B}, C) \) (equation \eqref{eq:krylov}) is computed with a matrix multiplication \( C \cdot K(\overline{A}, \overline{B}) \) before convolving with the input \( u \).

\cref{tab:complexity-full} provides more detailed complexity of this version of the LSSL with fixed \( A, \dt \).

Note that as mentioned in \cref{sec:discussion}, this cached algorithm is fairly fast, but the main drawback is that materializing the Krylov matrix \eqref{eq:krylov-matrix} requires \( O(NL) \) instead of \( O(L) \) space.

\subsection{Initialization of \texorpdfstring{\( A \)}{A}}

The LSSL initializes the \( A \) parameter in \eqref{eq:1} to the HiPPO-LegS operator,
which was derived to solve a particular continuous-time memorization problem.
This matrix \( A \in \mathbb{R}^{N \times N} \) is
\begin{align*}
  A_{nk}
  &=
  \begin{cases}
    (2n+1)^{1/2}(2k+1)^{1/2} & \mbox{if } n > k \\
    n+1 & \mbox{if } n = k \\
    0 & \mbox{if } n < k
  \end{cases}
  .
\end{align*}

Note that the LSSL-f is the LSSL with a non-trainable \( A \) (and \( \dt \)),
so that \( A \) is fixed to the above matrix.

\subsection{Initialization of \texorpdfstring{\( \dt \)}{dt}}

One distinction between the LSSL and the most related prior work is that the inclusion of the projection \eqref{eq:2} makes the layer a 1-dimensional to 1-dimensional map, instead of 1-D to \( N \)-D~\citep{voelker2019legendre,gu2020hippo}.
This enables us to concatenate \( H \) copies of this map (at the expense of computation, cf.\ \cref{sec:sllssl,sec:sllssl-proofs}).
Even when \( \dt \) is not trained as in the LSSL-f, these \( H \) copies allow multiple timescales to be considered by setting \( \dt \) differently for each copy.

In particular, we initialize \( \dt \) log-uniformly in a range \( \dt_{min}, \dt_{max} \) (i.e., \( \dt \) is initialized within this range, such that \( \log \dt \) is uniformly distributed).
The maximum and minimum values were generally chosen to be a factor of \( 100 \) apart such that the length of the sequences in the dataset are contained in this range.
Specific values for each model and dataset are in
\cref{sec:experiments-full}.
We did not search over these as a hyperparameter, but we note that it can be tuned for additional performance improvements in our experiments.

\subsection{Deep Neural Network Architecture}

The Deep LSSL models used in our experiments simply stack together LSSL layers in a simple deep neural network architecture.
We note the following architecture details.

\paragraph{Channels.}
The state-space model \eqref{eq:1}+\eqref{eq:2} accepts a 1-dimensional input \( u \), but does not strictly have to return a 1-dimensional output \( y \).
By making the matrices in \eqref{eq:2} dimension \( C \in \mathbb{R}^{M \times N}, D \in \mathbb{R}^{M \times 1} \),
the output \( y \) will be dimension \( M \) instead of \( 1 \).

We call \( M \) the number of \emph{channels} in the model.

\paragraph{Feedforward.}

There are two drawbacks with the current definition of LSSL:
\begin{itemize}%
  \item They are defined by running \( H \) independent copies of a state-space model, which means the \( H \) input features do not interact at all.
  \item If the channel dimension is \( M > 1 \), then the LSSL is a map from dimension \( 1 \) to \( M \), which means residuals cannot be applied.
\end{itemize}
These are both addressed by introducing a position-wise feedforward layer after the LSSL of shape \( H\cdot M \to H \).
This simultaneously mixes the hidden features, and projects the output back to dimension \( 1 \) if necessary.
There is also an optional non-linearity in between the LSSL and this feedforward projection; we fix it to the GeLU activation function in our models. %

We note that this factorization of parallel convolutions on the \( H \) features followed by a position-wise linear map is very similar to depth-wise separable convolutions~\citep{chollet2017xception}.

\paragraph{Residuals and normalization.}

To stack multiple layers of LSSLs together, we use very standard architectures for deep neural networks.
In particular, we use residual connections and a layer normalization (either pre-norm or post-norm) in the style of standard Transformer architectures.
Whether to use pre-norm or post-norm was chosen on a per-dataset basis, and depended on whether the model overfit; recent results have shown that pre-norm architectures are more stable \citep{liu2020understanding,davis2021catformer}, so we used it on harder datasets with less overfitting.
We note that we could have additionally inserted MLP modules in between LSSL layers, in the style of Transformers~\citep{vaswani2017attention}, but did not experiment with this.

\paragraph{Parameter count.}
The overall parameter count of an LSSL model is \( M \cdot H \cdot (H+N) \).

We primarily used two model sizes in our experiments, which were chosen simply to produce round numbers of parameters:
\begin{itemize}%
  \item LSSL small (\( \approx 200K \) parameters): \( 6 \) layers, \( H = 128, N = 128, M = 1 \).
  \item LSSL large (\( \approx 2M \) parameters): \( 4 \) layers, \( H = 256, N = 256, M = 4 \).
\end{itemize}
We did not search over additional sizes, but for some datasets reduced the model size for computational reasons.

\section{LSSL Proofs}
\label{sec:lssl-proofs}

This section gives refinements of the statements in \cref{sec:lssl}, additional results, and proofs of all results.

\cref{sec:lssl:ode} has a more detailed (and self-contained) summary of basic methods in ODE approximation which will be used in the results and proofs.

\cref{sec:lssl:proofs} give more general statements and proofs of \cref{lmm:gates} and \cref{lmm:picard} in \cref{lmm:rnn-dynamics} and \cref{thm:lssl-dynamics}, respectively.

\subsection{Approximations of ODEs}
\label{sec:lssl:ode}

We consider the standard setting of a first-order initial value problem (IVP) ordinary differential equation (ODE) for a continuous function \( f(t, x) \)
\begin{equation}
  \label{eq:ivp}
  \begin{aligned}%
    \dot{x}(t) &= f(t, x(t)) \\
    x(t_0) &= x_0 \\
  \end{aligned}
  .
\end{equation}
This differential form has an equivalent integral form
\begin{equation}
  \label{eq:integral}
  x(t) = x_0 + \int_{t_0}^t f(s, x(s)) \dd s
  .
\end{equation}

\cref{sec:lssl:picard,sec:lssl:euler} overview the Picard theorem and first-order numerical integration methods,
which apply to any IVP \eqref{eq:ivp}.
\cref{sec:lssl:gbt} then shows how to specialize it to linear systems as in equation \eqref{eq:1}.

At a high level, the basic approximation methods considered here use the integral form \eqref{eq:integral} and approximate the integral in the right-hand side by simple techniques.

\subsubsection{Picard Iteration}
\label{sec:lssl:picard}

The \textbf{Picard-Lindel\"of Theorem} gives sufficient conditions for the existence and uniqueness of solutions to an IVP.
As part of the proof, it provides an iteration scheme to compute this solution.

\begin{theorem}[Picard-Lindel\"of]%
  \label{thm:picard-iteration}
  In the IVP \eqref{eq:ivp}, if there is an interval around \( t_0 \) such that \( f \) is Lipschitz in its second argument,
  then there is an open interval \( I \ni t_0 \) such that there exists a unique solution \( x(t) \) to the IVP in \( I \).
  Furthermore, the sequence of \textbf{Picard iterates} \( x^{(0)}, x^{(1)}, \dots \) defined by
  \begin{align*}
    x^{(0)}(t) &= x_0
    \\
    x^{(\ell)}(t) &= x_0 + \int_{t_0}^t f(s, x^{(\ell-1)}(s)) \dd s
  \end{align*}
  converges to \( x \).
\end{theorem}

The Picard iteration can be viewed as approximating \eqref{eq:integral} by holding the previous estimate of the solution \( x^{(\ell-1)} \) fixed inside the RHS integral.

\subsubsection{Numerical Integration Methods}
\label{sec:lssl:euler}
Many methods for numerical integration of ODEs exist, which calculate discrete-time approximations of the solution.
We discuss a few of the simplest methods, which are first-order methods with local error \( O(h^2) \) \citep{butcher2008numerical}.

These methods start by discretizing \eqref{eq:integral} into the form
\begin{equation}
  \label{eq:integral-discrete}
  x(t_{k}) - x(t_{k-1}) = \int_{t_{k-1}}^{t_{k}} f(s, x(s)) \dd s.
\end{equation}
Here we assume a sequence of discrete times \( t_0, t_1, t_2, \dots \) is fixed.
For convenience, let \( x_k \) denote \( x(t_k) \) and let \( \dt_k := t_k - t_{k-1} \).
The goal is now to approximate the integral in the RHS of \eqref{eq:integral-discrete}.

\paragraph{Euler method.}
The Euler method approximates \eqref{eq:integral-discrete} by holding the left endpoint constant throughout the integral (i.e., the ``rectangle rule'' with left endpoint), \( f(s, x(s)) \approx f(t_{k-1}, x(t_{k-1})) \).
The discrete-time update becomes
\begin{equation}%
  \label{eq:euler-forward}
  \begin{aligned}%
    x_{k} - x_{k-1} &= (t_{k}-t_{k-1}) f(t_{k-1}, x(t_{k-1}))
    \\
    &=
    \dt_k f(t_{k-1}, x_{k-1})
    .
  \end{aligned}
\end{equation}

\paragraph{Backward Euler method.}
The backward Euler method approximates \eqref{eq:integral-discrete} by holding the right endpoint constant throughout the integral (i.e., the ``rectangle rule'' with right endpoint), \( f(s, x(s)) \approx f(t_k, x(t_k)) \).
The discrete-time update becomes
\begin{equation}%
  \label{eq:euler-backward}
  \begin{aligned}%
    x_{k} - x_{k-1} &= (t_{k}-t_{k-1}) f(t_k, x(t_k))
    \\
    &=
    \dt_k f(t_k, x_k)
    .
  \end{aligned}
\end{equation}

\subsubsection{Discretization of State-Space Models}
\label{sec:lssl:gbt}

In the case of a linear system, the IVP is specialized to the case
\begin{align*}
  f(t, x(t)) = A x(t) + B u(t).
\end{align*}
Note that here \( u \) is treated as a fixed external input, which is constant from the point of view of this ODE in \( x \).
Let \( u_k \) denote the average value in each discrete time interval,
\begin{align*}
  u_k = \frac{1}{\dt_k} \int_{t_{k-1}}^{t_k} u(s) \dd s
  .
\end{align*}

The integral equation \eqref{eq:integral-discrete} can be specialized to this case, and more generally a convex combination of the left and right endpoints can be taken to approximate the integral,
weighing them by \( 1-\alpha \) and \( \alpha \) respectively.
Note that the case \( \alpha=0, 1 \) are specializations of the forward and backward Euler method,
and the case \( \alpha = \frac{1}{2} \) is the classic ``trapezoid rule'' for numerical integration.
\begin{align*}
  x(t_{k}) - x(t_{k-1}) &= \int_{t_{k-1}}^{t_k} A x(s) \dd s + \int_{t_{k-1}}^{t_k} B u(s) \dd s
  \\ &=
  \int_{t_{k-1}}^{t_k} A x(s) \dd s + \dt_k B u_k
  \\ &\approx
  \dt_k \left[ (1-\alpha) A x_{k-1} + \alpha A x_k \right] + \dt_k B u_k
  .
\end{align*}
Rearranging yields
\begin{align*}
  (I - \alpha \dt_k \cdot A) x_k
  &= (I + (1-\alpha) \dt_k \cdot A) x_{k-1} + \dt_k \cdot B u_k
  \\
  x_k
  &= (I - \alpha \dt_k \cdot A)^{-1} (I + (1-\alpha) \dt_k \cdot A) x_{k-1} + (I - \alpha \dt_k \cdot A)^{-1} \dt_k \cdot B u_k
\end{align*}

This derives the \textbf{generalized bilinear transform (GBT)}~\citep{zhang2007performance}.
The \textbf{bilinear method} is the case \( \alpha=\frac{1}{2} \) of special significance, and was numerically found to be better than the forward and backward Euler methods \( \alpha = 0, 1 \) both in synthetic function approximation settings and in end-to-end experiments \citep[Figure 4]{gu2020hippo}.

\subsection{RNNs are LSSLs\texorpdfstring{: Proof of Results in \cref{sec:lssl-expressivity}}{}}
\label{sec:lssl:proofs}

We provide more detailed statements of \cref{lmm:gates,lmm:picard} from \cref{sec:lssl-expressivity}.
In summary, LSSLs and popular families of RNN methods all approximate the same continuous-time dynamics
\begin{equation}%
  \label{eq:dynamics}
 \dot{x}(t) = -x + f(t, x(t))
\end{equation}
by viewing them with a combination of two techniques.

We note that these results are about two of the most commonly used architecture modifications for RNNs.
First, the gating mechanism is ubiquitous in RNNs, and usually thought of as a heuristic for smoothing optimization~\citep{lstm}.
Second, many of the effective large-scale RNNs use linear (gated) recurrences and deeper models, which is usually thought of as a heuristic for computational efficiency~\citep{bradbury2016quasi}.
Our results suggest that neither of these are heuristics after all, and arise from standard ways to approximate ODEs.

To be more specific, we show that:
\begin{itemize}%
  \item Non-linear RNNs discretize the dynamics \eqref{eq:dynamics} by applying backwards Euler discretization to the linear term, which arises in the gating mechanism of RNNs (\cref{sec:lssl:proofs:gates}, \cref{lmm:rnn-dynamics}).
  \item A special case of LSSLs approximates the dynamics \eqref{eq:dynamics} (in continuous-time) by applying Picard iteration to the non-linear term (\cref{sec:lssl:proofs:picard}, \cref{thm:lssl-dynamics}).
  \item Deep linear RNNs approximate the dynamics \eqref{eq:dynamics} with both Picard iteration in the depth direction to linearize the non-linear term, and discretization (gates) in the time direction to discretize the equation (\cref{sec:lssl:proofs:qrnn}, \cref{thm:qrnn-dynamics}).
\end{itemize}
A comparison is summarized in \cref{tab:rnn}.

In the remainder of this section, we assume that there is an underlying function \( x(t) \) that satisfies \eqref{eq:dynamics} on some interval for any initial condition,
and that \( f \) is continuous and Lipschitz in its second argument.
Our goal is to show that several families of models approximate this in various ways.

\begin{table}
  \caption{A summary of the characteristics of popular RNN methods and their approximation mechanisms for capturing the dynamics \( \dot{x}(t) = -x(t) + f(t, x(t)) \) (equation \eqref{eq:dynamics}).
  The LSSL entries are for the very specific case with order \( N=1 \) and \( A=-1, B=1, C=1, D=0 \); LSSLs are more general.}
  \small
  \centering
  \begin{tabular}{@{}lllll@{}}
    \toprule
    Method                & RNN                                               & RNN                                                       & LSSL                                                 & LSSL                                 \\
    \emph{Variant}        & Gated                                             & Gated, linear                                             & Discrete \eqref{eq:1-discrete}+\eqref{eq:2-discrete} & Continuous \eqref{eq:1}+\eqref{eq:2} \\
    \emph{Special cases}  & LSTM~\citep{lstm}, GRU~\citep{chung2014empirical} & QRNN~\citep{bradbury2016quasi}, SRU~\citep{lei2017simple} &                                                      &                                      \\
    \midrule
    Deep?                 & Single-layer                                      & Deep                                                      & Deep                                                 & Deep                                 \\

    Continuous?           & Discrete-time                                     & Discrete-time                                             & Discrete-time                                        & Continuous-time                      \\
    Linear?               & Non-linear                                        & Linear                                                    & Linear                                               & Linear                               \\
    \midrule
    \emph{Approximation} \\
    \textbf{Depth-wise}   & -                                                 & Picard iteration                                          & Picard iteration                                     & Picard iteration                     \\
    \textbf{Time-wise}    & Backwards Euler                                   & GBT(\( \alpha=1 \))                                       & GBT(\( \alpha=\frac{1}{2} \)) (i.e.\ Bilinear)       & -                                    \\
    \bottomrule
  \end{tabular}
  \label{tab:rnn}
\end{table}

\subsubsection{Intuition / Proof Sketches}

We sketch the idea of how LSSLs capture popular RNNs. More precisely, we will show how approximating the dynamics \eqref{eq:dynamics} in various ways lead to types of RNNs and LSSLs.

The first step is to look at the simpler dynamics
\begin{align*}
  \dot{x}(t) = -x(t) + u(t)
\end{align*}
where there is some input \( u(t) \) that is independent of \( x \).
(In other words, in \eqref{eq:dynamics}, the function \( f(t, x) \) does not depend on the second argument.)

By directing applying the GBT discretization with \( \alpha=1 \), this leads to a gated recurrence (\cref{lmm:gates}).

The second step is that by applying the backwards Euler discretization more directly to \eqref{eq:dynamics},
this leads to a gated RNN where the input can depend on the state (\cref{lmm:rnn-dynamics}).

Alternatively, we can apply Picard iteration on \eqref{eq:dynamics}, which says that the iteration
\begin{align*}
  x^{(\ell)}(t) = x_0 + \int_{t_0}^t -x^{(\ell-1)}(s) \dd s + \int_{t_0}^t f(s, x^{(\ell-1)}(s)) \dd s
\end{align*}
converges to the solution \( x(t) \).

However, the first integral term is simple and can be tightened.
We can instead try to apply Picard iteration on only the second term, leaving the first integral in terms of \( x^{(\ell)} \).
Intuitively this should still converge to the right solution, since this is a weaker iteration; we're only using the Picard approximation on the second term.

\begin{align*}
  x^{(\ell)}(t) = x_0 + \int_{t_0}^t -x^{(\ell)}(s) \dd s + \int_{t_0}^t f(s, x^{(\ell-1)}(s)) \dd s
\end{align*}

Differentiating, this equation is the ODE
\begin{align*}
  \dot{x}^{(\ell)}(t) = -x^{(\ell)}(t) + f(t, x^{(\ell-1)}(t))
\end{align*}
This implies that alternating point-wise functions with a simple linear ODE \( \dot{x}^{(\ell)}(t) = -x^{(\ell)}(t) + u^{(\ell)}(t) \)
also captures the dynamics \eqref{eq:dynamics}.
But this is essentially what an LSSL is.

To move to discrete-time, this continuous-time layer can be discretized with gates as in \cref{lmm:gates}, leading to deep linear RNNs such as the QRNN,
or with the bilinear discretization, leading to the discrete-time LSSL.
We note again that in the discrete-time LSSL, \( \overline{A} \) and \( \overline{B} \) play the role of the gates \( \sigma, 1-\sigma \).

\subsubsection{Capturing gates through discretization}
\label{sec:lssl:proofs:gates}

\begin{lemma}%
  \label{lmm:rnn-dynamics}
  Consider an RNN of the form
  \begin{equation}
    \label{eq:rnn}
    x_k = (1-\sigma(z_k)) x_{k-1} + \sigma(z_k) \overline{f}(k, x_{k-1}),
  \end{equation}
  where \( \overline{f}(k, x) \) is an arbitrary function that is Lipschitz in its second argument (e.g., it may depend on an external input \( u_k \)).

  Then equation \eqref{eq:rnn} is a discretization of the dynamics \eqref{eq:dynamics}
  with step sizes \( \dt_k = \exp(z_k) \), i.e. \( x_k \approx x(t_k) \) where \( t_k = \sum_{i=1}^k \dt_i \).
\end{lemma}
\begin{proof}%
  Apply the backwards Euler discretization \eqref{eq:euler-backward} to equation \eqref{eq:dynamics} to get
  \begin{align*}
    x_k - x_{k-1} &= \dt_k \left[ -x_k + f(t_k, x_k) \right] \\
    (1 + \dt_k) x_k &= x_{k-1} + \dt_k f(t_k, x_k) \\
    x_k &= \frac{1}{1 + \dt_k} x_{k-1} + \frac{\dt_k}{1+\dt_k} f(t_k, x_k)
    .
  \end{align*}
  Note that \( \frac{\dt_k}{1+\dt_k} = \frac{e^{z_k}}{1+e^{z_k}} = \frac{1}{1+e^{-z_k}} \)
  and \( \frac{1}{1+\dt_k} = 1 - \frac{\dt_k}{1+\dt_k} \), thus
  \begin{align*}
    x_k &= (1-\sigma(z_k)) x_{k-1} + \sigma(z_k) \overline{f}(k, x_{k-1})
    .
  \end{align*}
  Here we are denoting \( \overline{f}(k, x) = f(t_k, x) \) to be a discrete-time version of \( f \) evaluatable at the given timesteps \( t_k \).
\end{proof}

Note that a potential external input function \( u(t) \) or sequence \( u_k \) is captured through the abstraction \( f(t, x) \).
For example, a basic RNN could define \( \overline{f}(k, x) = f(t_k, x) = \text{tanh}(W x + U u_k) \).

\subsubsection{Capturing non-linearities through Picard iteration}
\label{sec:lssl:proofs:picard}

The main result of this section is \cref{thm:lssl-dynamics} showing that LSSLs can approximate the same dynamics as the RNNs in the previous section.
This follows from a technical lemma.

\begin{lemma}%
  \label{lmm:lssl}
  Let \( f(t, x) \) be any function that satisfies the conditions of the Picard-Lindel\"of Theorem (\cref{thm:picard-iteration}).

  Define a sequence of functions \( x^{(\ell)} \) by alternating the (point-wise) function \( f \) with solving an ODE
  \begin{align*}
    x^{(0)}(t) &= x_0 \\
    u^{(\ell)}(t) &= f(t, x^{(\ell-1)}(t)) \\
    \dot{x}^{(\ell)}(t) &= A x^{(\ell)}(t) + u^{(\ell)}(t)
    .
  \end{align*}
  Then \( x^{(\ell)} \) converges to a solution \( x^{(\ell)}(t) \to x(t) \) of the IVP
  \begin{align*}
    \dot{x}(t) &= A x(t) + f(t, x(t))
    \\
    x(t_0) &= x_0.
  \end{align*}
\end{lemma}

\begin{theorem}%
  \label{thm:lssl-dynamics}
  A (continuous-time) deep LSSL with order \( N=1 \) and \( A=-1, B=1, C=1, D=0 \) approximates the non-linear dynamics \eqref{eq:dynamics}.
\end{theorem}
\begin{proof}%
  Applying the definition of an LSSL (equations \eqref{eq:1}+\eqref{eq:2}) with these parameters results in a layer mapping \( u(t) \mapsto y(t) \) where \( y \) is defined implicitly through the ODE
  \begin{align*}
    \dot{y}(t) = -y(t) + u(t)
    .
  \end{align*}
  This can be seen since the choice of \( C, D \) implies \( y(t) = x(t) \) and the choice of \( A, B \) gives the above equation.

  Consider the deep LSSL defined by alternating this LSSL with position-wise (in time) non-linear functions
  \begin{align*}
    u^{(\ell)}(t) &= f(t, y^{(\ell-1)}(t))
    \\
    \dot{y}^{(\ell)}(t) &= -y^{(\ell)}(t) + u^{(\ell)}(t)
    .
  \end{align*}
  But this is exactly a special case of \cref{lmm:lssl},
  so that we know \( y^{(\ell)}(t) \to y(t) \) such that \( y(t) \) satisfies
  \begin{align*}
    \dot{y}(t) = -y(t) + f(t, y(t))
  \end{align*}
  as desired.
\end{proof}

\begin{proof}[Proof of \cref{lmm:lssl}]%
  Let
  \begin{align*}
    z(t) = e^{-At} x(t)
  \end{align*}
  (and \( z_0 = z(t_0) = x(t_0) = x_0 \)).
  Note that
  \begin{align*}
    \dot{z}(t) &= e^{-At}\left[ \dot{x}(t) - A x(t) \right]
    \\&= e^{-At} f(t, x(t))
    \\&= e^{-At} f(t, e^{At} z(t))
    .
  \end{align*}
  Since \( f \) satisfies the conditions of the Picard Theorem (i.e., is continuous in the first argument and Lipschitz in the second),
  so does the function \( g \) where \( g(t, x) := e^{-At} f(t, e^{At} x) \) for some interval around the initial time.

  By \cref{thm:picard-iteration}, the iterates \( z^{(\ell)} \) defined by
  \begin{equation}
    \label{eq:tilted-picard}
    z^{(\ell)}(t) = z_0 + \int_{t_0}^t e^{-As} f(s, e^{As} z^{(\ell-1)}(s)) \dd s
  \end{equation}
  converges to \( z \).

  Define \( x^{(\ell)}(t) = e^{At} z^{(\ell)}(t) \).
  Differentiate \eqref{eq:tilted-picard} to get
  \begin{align*}
    \dot{z}^{(\ell)}(t)
    &= e^{-At} f(t, e^{At} z^{(\ell-1)}(t))
    \\
    &= e^{-At} f(t, x^{(\ell-1)}(t))
    \\
    &= e^{-At} u^{(\ell)}(t)
    .
  \end{align*}
  But
  \begin{align*}
    \dot{z}^{(\ell)}(t) = e^{-At}\left[ \dot{x}^{(\ell)}(t) - A x^{(\ell)}(t) \right],
  \end{align*}
  so
  \begin{align*}
    \dot{x}^{(\ell)}(t) = A x^{(\ell)}(t) + u^{(\ell)}(t).
  \end{align*}

  Since \( z^{(\ell)} \to z \) and \( x^{(\ell)}(t) = e^{At}z^{(\ell)}(t) \) and \( x(t) = e^{At} z(t) \),
  we have \( x^{(\ell)} \to x \).
\end{proof}

\subsubsection{Capturing Deep, Linear, Gated RNNs}
\label{sec:lssl:proofs:qrnn}

We finally note that several types of RNNs exist which were originally motivated by approximating linearizing gated RNNs for speed.
Although these were treated as a heuristic for efficiency reasons,
they are explained by combining our two main technical results.

\cref{lmm:rnn-dynamics} shows that a single-layer, discrete-time, non-linear RNN approximates the dynamics \eqref{eq:dynamics} through discretization, which arises in the gating mechanism.

\cref{thm:lssl-dynamics} shows that a deep, continuous-time, linear RNN approximates \eqref{eq:dynamics} through Picard iteration, where the non-linearity is moved to the depth direction.

Combining these two results leads to \cref{thm:qrnn-dynamics}, which says that a deep, discrete-time, linear RNN can also approximate the same dynamics \eqref{eq:dynamics}.

\begin{corollary}%
  \label{thm:qrnn-dynamics}
  Consider a deep, linear RNN of the form
  \begin{align*}
    x^{(\ell)}_k &= (1-\sigma(z_k)) x^{(\ell)}_{k-1} + \sigma(z_k) u^{(\ell)}_k
    \\
    u^{(\ell)}_k &= \overline{f}(k, x^{(\ell-1)}_{k})
    .
  \end{align*}
  This is a discretization of the dynamics \eqref{eq:dynamics}
  with step sizes \( \dt_k = \exp(z_k) \), i.e. \( x_k \approx x(t_k) \) where \( t_k = \sum_{i=1}^k \dt_i \).
\end{corollary}
\begin{proof}%
  By \cref{lmm:rnn-dynamics}, the first equation is a discretization of the continuous-time equation
  \begin{align*}
    \dot{x}^{(\ell)}(t) = -x^{(\ell)}(t) + u^{(\ell)}(t)
  \end{align*}
  where
  \begin{align*}
    u^{(\ell)}(t) = f(t, x^{(\ell-1)}(t))
  \end{align*}
  uses the continuous-time version \( f \) of \( \overline{f} \).
  But by \cref{lmm:lssl}, this is an approximation of the dynamics \eqref{eq:dynamics} using Picard iteration.
\end{proof}

Notable examples of this type of model include the Quasi-RNN or QRNN \citep{bradbury2016quasi} and the Simple Recurrent Unit (SRU) \citep{lei2017simple},
which are among the most effective models in practice.
We remark that these are the closest models to the LSSL and suggest that their efficacy is a consequence of the results of this section,
which shows that they are not heuristics.

We note that there are many more RNN variants that use a combination of these gating and linearization techniques that were not mentioned in this section, and can be explained similarly.

\section{LSSL Proofs and Algorithms}
\label{sec:sllssl-proofs}

This section proves the results in \cref{sec:mlssl}, and is organized as follows:
\begin{itemize}%
  \item \cref{sec:sllssl:hippo} gives a self-contained synopsis of the HiPPO framework~\citep{gu2020hippo}.
  \item \cref{sec:sllssl:general} proves \cref{thm:hippo}, which shows that the \( \hippo \) operators for any measure lead to a simple linear ODE of the form of equation \eqref{eq:1}.
  \item \cref{sec:sllssl:jacobi} proves \cref{thm:jacobi}, including a formal definition of quasiseparable matrices (i.e., how LSSL matrices are defined) in \cref{def:quasi}.
\end{itemize}

\paragraph{Notation}
This section is technically involved and we adopt notation to simplify reasoning about the shapes of objects.
In particular, we use bold capitals (e.g.\ \( \vA \)) to denote matrices and bold lowercase (e.g.\ \( \vb \)) to denote vectors.
For example, equation \eqref{eq:1} becomes \( \dot{x} = \vA x + \vb u \).
These conventions are adopted throughout \cref{sec:sllssl-proofs,sec:sllssl-algorithms}.

\subsection{Preliminaries: HiPPO Framework and Recurrence Width}
\label{sec:sllssl:hippo}

This section summarizes technical preliminaries taken directly from prior work.
We include this section so that this work is self-contained and uses consistent notation, which may deviate from prior work.
For example, we use modified notation from \citet{gu2020hippo} in order to follow conventions in control theory
(e.g., we denote input by \( u \) and state by \( x \) as in \eqref{eq:1}).

\cref{sec:hippo-definition} formally defines the HiPPO operator mathematically as in \citep[Section 2.2]{gu2020hippo},
and \cref{sec:hippo-framework} overviews the steps to derive the HiPPO operator as in \citep[Appendix C]{gu2020hippo}.
\cref{sec:rw} defines the class of Low Recurrence Width (LRW) matrices, which is the class of matrices that our generalization of the HiPPO results (\cref{thm:hippo}) uses.

\subsubsection{Definition of HiPPO Operator}
\label{sec:hippo-definition}

\begin{definition}[\citep{gu2020hippo}, Definition 1]
  Given a time-varying measure $\mu^{(t)}$ supported on $(-\infty,t]$, an N-dimensional subspace $\calG$ of polynomials, and a continuous function $u:\R_{\geq 0} \rightarrow \R$, HiPPO defines a \emph{projection} operator $\mathsf{proj}_t$ and a \emph{coefficient extraction} operator $\mathsf{coef}_t$ at every time $t$, with the following properties:
  \begin{enumerate}
    \item $\mathsf{proj}_t$ takes a function $u$ restricted up to time $t$, $u_{\leq t} := u(x)|_{x \leq t}$, and maps it to a polynomial $g^{(t)} \in \calG$, that minimizes the approximation error $\|u_{\leq t} - g^{(t)}\|_{L_2(\mu^{(t)})}$.
    \item $\mathsf{coef}_t:\calG \rightarrow \R^N$ maps the polynomial $g^{(t)}$ to the coefficients $c(t) \in \R^N$ of the basis of orthogonal polynomials defined with respect to the measure $\mu^{(t)}$.
  \end{enumerate}

  The composition $\mathsf{coef}_t \circ \mathsf{proj}_t$ is called \( \hippo \), which is an operator mapping a function $u: \R_{\geq 0} \rightarrow \R$ to the optimal projection coefficients $c:\R_{\geq 0} \rightarrow \R^N$ (i.e $(\hippo(u))(t) = \mathsf{coef}_t(\mathsf{proj}_t(f))$.
  \label{def:hippo}
\end{definition}

\subsubsection{HiPPO Framework for Deriving the HiPPO Operator}
\label{sec:hippo-framework}

The main ingredients of HiPPO consists of an approximation measure and an orthogonal polynomial basis. We recall how they are defined in \cite{gu2020hippo} (we note that compared to \citet{gu2020hippo}, our notation has changed from input \( f(t) \) coefficients (state) \( c(t) \) to input \( u(t) \) and coefficients (state) \( x(t) \), following conventions in controls).

\paragraph{Approximation Measures} At every $t$, the approximation quality is defined with respect to a measure $\mu^{(t)}$ supported on $(-\infty,t]$. We assume that the measures $\mu^{(t)}$ have densities $\omega(t,Y) := \frac{\dd \mu^{(t)}}{ \dd Y}$. Note that this implies that integrating with respect to $\dd \mu^{(t)}$ is the same as integrating with respect to $\omega(t,Y) \dd Y$.
\paragraph{Orthogonal Polynomial basis} Let $\{P_n^{(t)}\}_{n \in \mathbb{N}}$ denote a sequence of orthogonal polynomials with respect to some time-varying measure $\mu^{(t)}$. Let $p_n^{(t)}$ be the normalized version of of orthogonal $P_n^{(t)}$, and define \[p_n(t,Y) = p_n^{(t)}(Y).\]
In particular, the above implies that
\[\int_{-\infty}^t p_n^{(t)}(Y)\cdot p_m^{(t)}(Y) \omega(t,Y) \dd Y =\delta_{m,n}.\]

In the general framework, HiPPO does not require an orthogonal polynomial basis as the selected basis.
The choice of basis is generalized by tilting with $\chi$.

\paragraph{Tilted measure and basis} For any scaling function $\chi(t,Y)$, the functions $p_n(t,Y)\chi(t,Y)$ are orthogonal with respect to the density $\frac{\omega}{\chi^2}$ at every time $t$.
Define $\nu(t)$ to be the normalized measure with density proportional to $\frac{\omega}{\chi^2}$, with normalization constant $\zeta(t) = \int_0^t \frac{\omega(t,Y)}{\chi(t,Y)^2}\mathrm{d}x$.

We express the coefficients $x_n(t)$ calculated by the HiPPO framework as:

\begin{equation}
  x_n(t) = \frac{1}{\sqrt{\zeta(t)}}\int_0^t u(Y)p_n(t,Y) \frac{\omega(t,Y)}{\chi(t,Y)} \dd Y.
  \label{eq:gen-cnt}
\end{equation}

To use this to derive $\dot x_n(t)$, let $h(t,Y) = u(Y)p_n(t,Y)\omega(t,Y)$. We see that

\begin{align*}
  \dot x_n(t) &=
  \frac{\mathrm{d}}{\mathrm{d}t} \int_0^t u(Y)p_n(t,Y)\omega(t,Y)\mathrm{d}Y \\
  &= \frac{\mathrm{d}}{\mathrm{d}t} \int_0^t h(t,Y)\mathrm{d}Y \\
  &= \int_{0}^{t} \frac{\partial}{\partial t} h(t,Y)\mathrm{d}Y  + h(t,t) \\
  &= \int_0^t u(Y)  \parens{\frac{\partial}{\partial t}p_n\parens{t,Y}}\omega\parens{t,Y}\mathrm{d}Y + \int_0^t f(Y) p_n\parens{t,Y}\parens{\frac{\partial}{\partial t}\omega\parens{t,Y}}\mathrm{d}Y \\&\ \ \ \ + u(t)p_n(t,t)\omega(t,t).
\end{align*}

This allows $\dot x_n(t)$ to be written as

\begin{align}
    \dot x_n(t) &= u(t)p_n(t,t)\omega(t,t) + \int_0^t u(Y)  \parens{\frac{\partial}{\partial t}p_n\parens{t,Y}}\omega\parens{t,Y}\mathrm{d}Y \nonumber \\&\ \ \ \ +  \int_0^t f(Y) p_n\parens{t,Y}\parens{\frac{\partial}{\partial t}\omega\parens{t,Y}}\mathrm{d}Y.
  \label{eq:ddt-cnt}
\end{align}

Although \citet{gu2020hippo} describe the framework in the full generality above and use \( \chi \) as another degree of freedom, in their concrete derivations they always fix \( \chi = \omega \).
Our general results also use this setting.
For the remainder of this section, we assume the ``full tilting'' case \( \chi = \omega \).
In particular, this means that in \cref{eq:ddt-cnt}, we essentially substitute $\omega$ above with $1$ and divide each term by the inverse square root of our normalization constant, $\zeta$, to get the coefficient dynamics that we will use in our arguments:

\begin{equation}
  \dot x_n(t) = \frac{1}{\sqrt{\zeta(t)}}u(t)p_n(t,t) + \frac{1}{\sqrt{\zeta(t)}}\int_0^t u(Y)  \parens{\frac{\partial}{\partial t}p_n\parens{t,Y}}\mathrm{d}Y 
  \label{eq:gen-ddt-cnt}
\end{equation}

Now, if we can show that each of the integrated terms in (\ref{eq:ddt-cnt}) are linear combinations of $x_n(t)$, this would be the same as saying that $\dot x_n(t) = \vA(t) x(t) + \vb(t)u(t)$ for some $\vA(t)$. Therefore, the incremental update operation would be bounded by the runtime of the matrix-vector operation $\vA(t) x(t)$.

\subsubsection{Recurrence Width}
\label{sec:rw}

Our final goal is to show that
 $\dot x_n(t) = \vA(t) x(t) + \vb(t)u(t)$ for some $\vA(t)$ with constant recurrence width (see Definition \ref{def:recur-width}).
This will show \cref{thm:hippo}, and also imply that the MVM $\vA(t)x(t)$ can be computed in $\tilde{O}(N)$ time.
To build this argument, we borrow the fact that OPs all have recurrence width 2 and results regarding matrix-vector multiplication of matrices with constant recurrence width along with their inverses.

 \begin{definition}[\cite{de2018two}]
An $N \times N$ matrix $\vA$ has recurrence width $t$ if the polynomials $a_i(X) = \sum_{j=0}^{N-1}\vA[i,j]X^j$ satisfy $\deg(a_i) \leq i$ for $i < t$, and
\[a_i(X) = \sum_{j=1}^t g_{i,j}(X)a_{i-j}(X)\]
for $i \geq t$, where the polynomials $g_{i,j} \in \R[X]$ have degree at most $j$.
\label{def:recur-width}
\end{definition}

\begin{theorem}[\cite{de2018two}, Theorem 4.4]
For any $N \times N$ matrix $\vA$ with constant recurrence width, any vector $\vx \in \R^n$, $\vA\vx$ can be computed with $\tilde{O}(N)$ operations over $\R$.
 \label{thm:nlin-matrixvec}
\end{theorem}

\begin{theorem}[\cite{de2018two}, Theorem 7.1]
For any $N \times N$ matrix $\vA$ with constant recurrence width, any vector $\vx \in \R^n$, $\vA^{-1}\vx$ can be computed with $\tilde{O}(N)$ operations over $\R$.
 \label{thm:nlin-matrixvecinv}
\end{theorem}

For the rest of the note we'll assume that any operation over $\R$ can be done in constant time. It would be useful for us to define $\vP \in \R^{N \times N}$ such that the coefficients of the OP $p_i(X)$, i.e. $p_i(X) = \sum_{j=0}^{N-1} \vP[i,j]X^j$.

\subsection{Proof of \texorpdfstring{\cref{thm:hippo}}{Theorem 1}}
\label{sec:sllssl:general}

This section proves \cref{thm:hippo}, which is restated formally in \cref{cor:A_1_structure}.
\cref{sec:hippo:op} proves some results relating orthogonal polynomials to recurrence width (\cref{sec:rw}).
\cref{sec:hippo:general} proves \cref{cor:A_1_structure}.
\cref{sec:hippo:translated,sec:hippo:scaled} provides examples showing how \cref{cor:A_1_structure} can be specialized to exactly recover the HiPPO-LegT, HiPPO-LagT, HiPPO-LegS methods \citep{gu2020hippo}.

\subsubsection{Relating Orthogonal Polynomials and Recurrence Width}
\label{sec:hippo:op}

Next we introduce the following lemma, which will be useful in our arguments:

\begin{lemma}
 For any $n$, there exists ordered sets of coefficients $\alpha_n = \{\alpha_{n,i}\},\beta_n = \{\beta_{n,i}\}$, %
 \begin{enumerate}[label=(\roman*)]
     \item $p'_n(Z) = \sum_{i=0}^{n-1} \alpha_{n,i}p_i(Z)$
     \item $Zp'_n(Z) = \sum_{i=0}^{n-1}\beta_{n,i}p_i(Z)$
 \end{enumerate}
\label{lmm:master-coeffs}
\end{lemma}
\begin{proof}
Follows from the fact that $p_i(z)$ for $0 \leq i < N$ forms a basis and the observation of the degrees of the polynomials on the LHS.
\end{proof}

The following matrices will aid in showing verifying that matrix vector multiplication with a given matrix $\vA$ can be computed in $O(\tilde{N})$ time.

\begin{definition}
$\vD_1,\vD_2 \in \R^{N \times N}$ are the matrices such that 
\[p'_i(Z) = \sum_{j=0}^{N-1}\vD_1[i,j]Z^j, \text{ and }\ Zp'_i(Z) = \sum_{j=0}^{N-1}\vD_2[i,j]Z^j.\]
\label{def:dimatrices}
\end{definition}

Let $\vS$ be the ``right shift" matrix, i.e. for any matrix $\vM$, $\vM\vS$ has the columns of $\vM$ shifted to right by one. Note that $\vS^T$ corresponds to the ``left shift" matrix.

We now note that:

\begin{lemma}
$\vD_1=\vP\cdot\mathrm{diag}(0,1,\dots,N-1)\cdot \vS^T$ and $\vD_2=\vP\cdot\mathrm{diag}(0,1,\dots,N-1)$. In particular,
$\vD_1\vz$ and $\vD_2\vz$ can be computed in $\tilde{O}(N)$ time for any $\vz \in \R^N$. 
\label{lmm:pbm1-recurlen}
\end{lemma}
\begin{proof}

Recall that $\vP$ has the coefficients of the OP polynomials $p_0(Z),\dots, p_{N-1}(Z)$ as its rows. Then note that
\begin{equation}
   p'_n(Z) = \sum_{i=0}^{n-1} i\cdot \vP[n,i] \cdot Z^{i-1}. 
\label{eq:derv-standard-basis}
\end{equation}
The claim on $\vD_1=\vP\cdot\mathrm{diag}(0,1,\dots,N-1)\cdot \vS^T$ follows from the above. 
  Recall that $\vD$ has recurrence width of $2$. The claim on the runtime of computing $\vD_1\vz$ then follows from 
  Theorem \ref{thm:nlin-matrixvec} and the fact that both $\mathrm{diag}(0,1,\dots,N-1)$ and $\vS^T$ is $n$-sparse.

From \cref{eq:derv-standard-basis}, it is easy to see that
\[Zp'_n(Z) = \sum_{i=0}^{n-1} i\cdot \vP[n,i] \cdot Z^{i}.\]
The claim on the structure of $\vD_2$ then follows from the above expression. The claim on runtime of computing $\vD_2\vz$ follows from essentially the same argument as for $\vD_1\vz$.

\end{proof}

Finally, we make the following observation:
\begin{lemma}
\label{lmm:alpha-beta-rw}
Let $\vA'$ and $\vB'$ be defined such that $\vA'[n,i]=\alpha_{n,i}$ and $\vB'[n,i]=\beta_{n,i}$. Then both $\vA'$ and $\vB'$ are both products of three matrices: two of which have recurrence width at most $2$ and the third is the inverse of a matrix that has recurrence width $2$.
\end{lemma}
\begin{proof}
We note that since $\vP$ expresses the orthogonal polynomials in standard basis, $\vP^{-1}$ changes from OP basis to standard basis. This along with \cref{lmm:pbm1-recurlen} implies that $\vA=\vP\cdot \parens{\mathrm{diag}(0,1,\dots,N-1)\cdot \vS^T}\cdot \vP^{-1}$. It is easy to check that $\mathrm{diag}(0,1,\dots,N-1)\cdot \vS^T$ has recurrence width $1$ and the claim on $\vA'$ follows since $\vP$ has recurrence width $2$. A similar argument proves the claim on $\vB'$.
\end{proof}
\subsubsection{HiPPO for General Measures}
\label{sec:hippo:general}

Let $\theta : \R_{\geq 0} \mapsto \R_{\geq 0}$ be a function such that for all $t$, $\theta(t) \leq t$ and $\theta(t)$ is differentiable.

In what follows, define \[z = \frac{2(\x-t)}{\theta(t)} + 1.\] We note that
\begin{equation}
  \frac{\mathrm{d}z}{\mathrm{d}\x} = \frac{2}{\theta(t)} \label{eq:dxdz}.
\end{equation}

Further, note that:
\begin{align*}
  \frac{\mathrm{d}z}{\mathrm{d}t} &= \frac{\mathrm{d}}{\mathrm{d}t} \parens{\frac{2(\x-t)}{\theta(t)}+1} \\
  &= -\frac{2}{\theta(t)} -  \frac{2(\x-t)\theta'(t)}{\theta^2(t)} \\
  &= -\frac{2}{\theta^2(t)}\parens{\theta(t) + (\x-t)\theta'(t)}. \\
\end{align*}

From the definition of $z$, we see that $\x - t = \frac{(z-1)\theta(t)}{2}$. Then

\begin{align}
  \frac{\mathrm{d}z}{\mathrm{d}t} &=
  -\frac{2}{\theta(t)} \parens{1 + \parens{\frac{z-1}{2}}\theta'(t)}  \nonumber \\
  &= -\frac{2}{\theta(t)} - \frac{(z-1)\theta'(t)}{\theta(t)}  \label{eq:dzdt}.
\end{align}

Additionally, given a measure $\omega$ on [-1,1] and OP family $p_0(\x), \ p_1(\x), \dots $ such that for all $i \neq j$, \[\int^1_{-1}p_i(\x)p_j(\x)\omega(\x)\mathrm{d}\x = \delta_{i,j},\] define

\[\omega(\x,t) = \frac{2}{\theta(t)}\omega(z) \text{ and } p_n(\x,t) = p_n(z). \]

Then we can adjust (\ref{eq:gen-cnt}) to:

\begin{equation}
  x_n(t) = \int^t_{t - \theta(t)} u(\x)p_n(z)\frac{2}{\theta(t)}\mathrm{d}\x.
  \label{eq:cnt}
\end{equation}

The Leibniz integral rule states that

\[\frac{\partial}{\partial t}\int_{\alpha(t)}^{\beta(t)} h(t,\x)\mathrm{d}\x =  \int_{\alpha(t)}^{\beta(t)} \frac{\partial}{\partial t} h(t,\x)\mathrm{d}\x - \alpha'(t)h(\alpha(t),t) + \beta'(t)h(t,t). \]

If we let $\alpha(t) = t-\theta(t)$ and $\beta(t) = t$, then applying the Leibniz rule to (\ref{eq:cnt}) we get:

\begin{align*}
  \dot x_n(t)  &= \int^t_{t - \theta(t)} u(\x)\frac{\partial}{\partial t}\parens{p_n(z)\frac{2}{\theta(t)}}\mathrm{d}\x - (1 - \theta'(t))u(t - \theta(t))p_n(t - \theta(t),t)\frac{2}{\theta(t)}  \\&\ \ \ \ + u(t)p_n(t,t)\frac{2}{\theta(t)} \\
  &= - (1 - \theta'(t))u(t - \theta(t))p_n(-1)\frac{2}{\theta(t)} + u(t)p_n(1)\frac{2}{\theta(t)}  \\ &  \ \ \ \ \ +\int^t_{t - \theta(t)} u(\x)\frac{\mathrm{d}z}{\mathrm{d} t}p'_n(z)\frac{2}{\theta(t)}\mathrm{d}\x - \frac{\theta'(t)}{\theta(t)}\int^t_{t - \theta(t)} u(\x)p_n(z)\frac{2}{\theta(t)}\mathrm{d}\x.
\end{align*}

From (\ref{eq:dzdt}), it follows that

\begin{align}
  \dot x_n(t) =  -&\frac{2(1 - \theta'(t))u(t - \theta(t))p_n(-1)}{\theta(t)} + \frac{2\cdot u(t)p_n(1)}{\theta(t)} -  \frac{2}{\theta(t)}\int^t_{t - \theta(t)} u(\x)p'_n(z)\frac{2}{\theta(t)}\mathrm{d}\x - \nonumber \\ &\frac{\theta'(t)}{\theta(t)}\int^t_{t - \theta(t)} u(\x)(z-1)p'_n(z)\frac{2}{\theta(t)}\mathrm{d}\x  - \frac{\theta'(t)}{\theta(t)}\int^t_{t - \theta(t)} u(\x)p_n(z)\frac{2}{\theta(t)}\mathrm{d}\x.
  \label{eq:legt-dcnt}
\end{align}

Because $\deg\parens{p'_n(z)} \leq n-1$ and $\deg\parens{(z-1)p'_n(z)} \leq n$, they can be written as a linear combination of $\{p'_i\}_{i\leq n}$. Let us define $\{\alpha_{n,j}\}$, $\{\beta_{n,j}\}$ such that
\begin{equation}
  p'_n(z) = \sum_{j=0}^{n-1} \alpha_{n,j}p_j(z) \text{ and } (z-1)p'_n(z) = \sum_{j=0}^n \beta_{n,j}p_j(z).
  \label{eq:z1alpha}
\end{equation}

Then by using (\ref{eq:z1alpha}) in (\ref{eq:legt-dcnt}), we get:
\begin{align*}
  \dot x_n(t) &=  -\frac{2(1 - \theta'(t))u(t + \theta(t))p_n(-1)}{\theta(t)} + \frac{2\cdot u(t)p_n(1)}{\theta(t)}
 \\&\ \ \ \  - \frac{2}{\theta(t)}\sum^{n-1}_{j=0}\alpha_{n,j}\int^t_{t - \theta(t)} u(\x) p_j(z)\frac{2}{\theta(t)}\mathrm{d}\x \\&\ \ \ \ - \frac{\theta'(t)}{\theta(t)}\sum^n_{j=0}\beta_{n,j}\int^t_{t - \theta(t)} u(\x) p_j(z)\frac{2}{\theta(t)}\mathrm{d}\x - \frac{\theta'(t)}{\theta(t)}\int^t_{t - \theta(t)} u(\x)p_n(z)\frac{2}{\theta(t)}\mathrm{d}\x\\
  &=  -\frac{2(1 - \theta'(t))u(t + \theta(t))p_n(-1)}{\theta(t)} + \frac{2\cdot u(t)p_n(1)}{\theta(t)} \\&\ \ \ \ - \frac{2}{\theta(t)}\sum^{n-1}_{j=0}\alpha_{n,j} x_j(t) -  \frac{\theta'(t)}{\theta(t)}\sum^n_{j=0}\beta_{n,j} x_j(t) - \frac{\theta'(t)}{\theta(t)} x_n(t).
\end{align*}

Thus, in vector form we get

\begin{theorem}
  \[\dot x_n(t) = - \frac{1}{\theta(t)}\vA_1(t)x(t) - \frac{2}{\theta(t)}(1 - \theta'(t))u(t - \theta(t))\begin{bmatrix} \vdots \\ p_n(-1) \\ \vdots \end{bmatrix} + \frac{2}{\theta(t)}u(t)\begin{bmatrix} \vdots \\ p_n(1) \\ \vdots \end{bmatrix}\]

  where $\vA_1(t)[n,k] = \begin{cases}
    2\alpha_{n,k} + \theta'(t)\beta_{n,k}  &\text{ if } k < n  \\
    \theta'(t)\beta_{n,n} +\theta'(t)  &\text{ if } k = n  \\ 0 \ \ &\text{ otherwise }
  \end{cases}$ for $\alpha_{n,k}$,$\beta_{n,k}$ as defined in (\ref{eq:z1alpha}).
  \label{thm:dctthetat}
\end{theorem}

\begin{corollary}
  The matrix $\vA_1$ in \cref{thm:dctthetat} can be re-written as
  \begin{equation}
      \vA_1 = 2\cdot \vA'+\theta'(t)\cdot \vB'+\theta'(t)\cdot \mathbf{I}.
      \label{eq:A_1-deomcpose}
  \end{equation}

  In particular, both $\vA'$ and $\vB'$ both products of three matrices: two of which have recurrence width at most $2$ and the third is the inverse of a matrix that has recurrence width $2$.
  \label{cor:A_1_structure}
\end{corollary}
\begin{proof}
\cref{eq:A_1-deomcpose} follows from \cref{thm:dctthetat} and defining $\vA'$ and $\vB'$ to contain the $\alpha_{n,k}$ and $\beta_{n,k}$ coefficients.
\end{proof}
\subsubsection{Translated HiPPO (Sliding Windows)}
\label{sec:hippo:translated}

The case when $\theta(t) = \theta$ for all $t$ represents a constant-size sliding window, which \citet{gu2020hippo} denote as the ``Translated HiPPO'' case with instantiations such as HiPPO-LegT (Translated Legendre) and HiPPO-LagT (Translated Laguerre).

We now state a corollary of Theorem \ref{thm:dctthetat} for the case of $\theta(t) = \theta$ for all t.
\begin{corollary}
  Let $\theta(t) = \theta$ for all $t$. Then

  \[\dot x_n(t) = -\frac{1}{\theta}\vA_1x(t) -\frac{2}{\theta}u(t - \theta)\begin{bmatrix} \vdots \\ p_n(-1) \\ \vdots \end{bmatrix} + \frac{2}{\theta}u(t)\begin{bmatrix} \vdots \\ p_n(1) \\ \vdots \end{bmatrix}. \]

  where $\vA_1[n,j] = \begin{cases}
    2\alpha_{n,k}  &\text{ if } k < n  \\  0 \ \ &\text{ otherwise }
  \end{cases}$.
  \label{cor:thetat-theta}
\end{corollary}

Next, we use the approximation \[u(x) \approx \sum_{k=0}^{N-1} x_k(t)p_k(z).\] to handle the $u(t - \theta)$ term in Corollary \ref{cor:thetat-theta}.

\begin{corollary}
  Let $\theta(t) = \theta$ for all $t$. Then

  \[\dot x_n(t) \approx - \frac{1}{\theta}\vA x(t) + \frac{2}{\theta}u(t)\begin{bmatrix} \vdots \\ p_n(1) \\ \vdots \end{bmatrix}\]

  where $\vA = \vA_1 + 2\vA_2$ for  $\vA_1$ as defined in Corollary \ref{cor:thetat-theta} and $\vA_2[n,k] =  p_n(-1)p_k(-1)$. 
  \label{cor:approx-ftt}
\end{corollary}
\begin{proof}

  To approximate $u(t - \theta)$, we note that when  $\x = t - \theta$, $z = -1$. Then
  \[u(t - \theta) \approx \sum_{k=0}^{N-1} x_k(t)p_k(-1).\]

  Then by Corollary \ref{cor:thetat-theta},

  \begin{align}
    \dot x_n(t) &\approx -\frac{1}{\theta}\vA_1x(t) - \frac{2}{\theta}\parens{\sum_{k=0}^{N-1} x_k(t)p_k(-1)}\begin{bmatrix} \vdots \\ p_n(-1) \\ \vdots \end{bmatrix} + \frac{2}{\theta}u(t)\begin{bmatrix} \vdots \\ p_n(-1) \\ \vdots \end{bmatrix}. \label{eq:theta-tsum}
  \end{align}

  Let us define a matrix, $\vA_2 \in \R^{N \times N}$ matrix such that $\vA_2[n,k] = p_n(-1)p_k(-1)$. Then the claim follows.

\end{proof}

We now show that the special case of Corollary \ref{cor:approx-ftt} for Legendre matches the results from \cite{gu2020hippo}.

\begin{corollary}
  \label{cor:HiPPO-Leg}
  Let  $p_n(z) = \parens{\frac{2n+1}{2}}^{1/2}P_n(z)$ where $P_n(z)$ are the Legendre polynomials. Then

  \begin{align*}
    \dot x_n(t) \approx \frac{1}{\theta}\vA x(t) + \frac{2}{\theta}\vb u(t)
  \end{align*}

  where \[\vA[n,k] = (2n+1)^{\frac{1}{2}}\parens{2k + 1}^{\frac{1}{2}}\begin{cases}
    1  &\text{ if } k \leq n  \\  (-1)^{n-k} \ \ &\text{ if } k \geq n
  \end{cases},
\] and $\vb[n] = \parens{\frac{2n+1}{2}}^{\frac{1}{2}}$. 
\end{corollary}
\begin{proof}

  From Corollary \ref{cor:approx-ftt}, \[\dot x_n(t) \approx - \frac{1}{\theta}\vA x(t) + \frac{2}{\theta}u(t)\begin{bmatrix} \vdots \\ p_n(1) \\ \vdots \end{bmatrix}\]

  where $\vA = \vA_1 + 2\vA_2$ for  $\vA_1$ as defined in Corollary \ref{cor:thetat-theta} and $\vA_2[n,k] =  p_n(-1)p_k(-1)$. 

  It is known from (7.21.1) and in \cite{szego1975orthogonal} that
  \begin{equation}
    \label{eq:leg-val-at+-1}
    p_n(-1) = \parens{\frac{2n+1}{2}}^{\frac{1}{2}}P_n(-1) \text{ and } P_n(-1) = (-1)^n.
  \end{equation}

  Further,
  \begin{equation}
    \label{eq:leg-val-at+-2}
    p_n(1) = \parens{\frac{2n+1}{2}}^{\frac{1}{2}}P_n(1) \text{ and } P_n(1) = 1.
  \end{equation}

  Then $\vb[n] = \parens{\frac{2n+1}{2}}^{\frac{1}{2}}$ follows from Corollary \ref{cor:approx-ftt} and (\ref{eq:leg-val-at+-2}).

  From the following recurrence relations \cite[Chapter 12]{ArfkenGeorgeB.GeorgeBrown2013Mmfp}:

  \begin{equation*}
    (2n+1)P_n(z) = P'_{n+1}(z) + P'_{n-1}(z)
    \label{eq:pnz-rec1}
  \end{equation*}

  implies that

  \begin{equation*}
    P'_{n+1}(z) = (2n+1)P_n(z) + (2n+1)P_{n-2}(z) + \dots +,
    \label{eq:pnz-rec2}
  \end{equation*}

  which in turn implies

  \[P'_n = (2n-1)P_{n-1}(z) + (2n - 5)P_{n-3}(z) + \dots. \]

  Then

  \begin{align*}
    p'_n(z) &=\parens{\frac{2n+1}{2}}^{\frac{1}{2}} \cdot P'_n(z) \\
    &= \parens{\frac{2n+1}{2}}^{\frac{1}{2}}\parens{\parens{2n-1}\parens{\frac{2}{2n-1}}^{\frac{1}{2}}P_{n-1}(z) + \parens{2n-5}\parens{\frac{2}{2n-5}}^{\frac{1}{2}}P_{n-3}(z) + \dots} \\ &= \parens{2n+1}^{\frac{1}{2}}\parens{\parens{2n-1}^{\frac{1}{2}}P_{n-1}(z) + \parens{2n-5}^{\frac{1}{2}}P_{n-3}(z) + \dots} .
  \end{align*}

  Thus, we have
  \[\alpha_{n,k} = \begin{cases} (2n + 1)^{\frac{1}{2}}(2k + 1)^{\frac{1}{2}} \text{ if } k < n \text{ and } n - k \text{ is odd,} \\  0 \ \ \text{ is otherwise.}\end{cases}, \]
  Recalling that $\vA_1[n,k] = 2\alpha_{n,k}$.

  We note that from~\eqref{eq:leg-val-at+-1}, $\vA_2[n,k] = \parens{\frac{2n+1}{2}}^{\frac{1}{2}}\parens{\frac{2k+1}{2}}^{\frac{1}{2}}(-1)^{n}(-1)^k = \frac{\parens{2n+1}^{\frac{1}{2}}\parens{2k+1}^{\frac{1}{2}}}{2}(-1)^{n-k}$.

  Recalling $\vA = \vA_1 + 2\vA_2$, we get:

  \[\vA[n,k] = (2n+1)^{\frac{1}{2}}\parens{2k + 1}^{\frac{1}{2}}\begin{cases}
    2+(-1)^{n-k}  &\text{ if } k < n  \ \ \text{ and } n - k \text{ is odd }\\  
    0+(-1)^{n-k} \ \ &\text{ if } k < n \ \  \text{ and } n - k \text{ is even }\\
    (-1)^{n-k} \ \ &\text{ if } k \geq n \\
  \end{cases}.
\]

Note that the above is the same as:

\[\vA[n,k] = (2n+1)^{\frac{1}{2}}\parens{2k + 1}^{\frac{1}{2}}\begin{cases}
  1  &\text{ if } k \leq n  \\  (-1)^{n-k} \ \ &\text{ if } k \geq n
\end{cases},
\]

which completes our claim.

\end{proof}

\subsubsection{Scaled HiPPO: Recovering HiPPO-LegS}
\label{sec:hippo:scaled}

We now use Theorem \ref{thm:dctthetat} to recover the HiPPO-LegS instantiation for the ``Scaled Legendre'' measure, the main method from \citet{gu2020hippo}.

\begin{corollary}
  \label{cor:HiPPO-LegS}
  Let  $p_n(z) = \parens{\frac{2n+1}{2}}^{1/2}P_n(z)$ where $P_n(z)$ are the Legendre polynomials and let $\delta(t) = t$ for all t. Then

  \begin{align*}
    \dot x_n(t) = \frac{1}{t}\vA x(t) + \frac{2}{t}\vb u(t)
  \end{align*}

  where \[\vA[n,k] = \begin{cases} (2n+1)^{\frac{1}{2}}\parens{2k + 1}^{\frac{1}{2}} &\text{ if } k < n \\
    n+1  &\text{ if } k = n  \\  0 \ \ &\text{ if } k >n
  \end{cases},
\] and $\vb[n] = \parens{\frac{2n+1}{2}}^{\frac{1}{2}}$. 
\end{corollary}
\begin{proof}

  Let $\theta(t) = t$. By Theorem \ref{thm:dctthetat} and noting that $\theta
  (t) = 1$, we get: \[\dot x_n(t)= - \frac{1}{t}\vA_1 x(t) + \frac{2}{t}u(t)\begin{bmatrix} \vdots \\ p_n(1) \\ \vdots \end{bmatrix}\]

  where
  \begin{equation}
    \vA_1(t)[n,k] = \begin{cases}
      2\alpha_{n,k} + \beta_{n,k}  &\text{ if } k < n  \\ 
      \beta_{n,n} +1  &\text{ if } k = n  \\ 0 \ \ &\text{ otherwise }
      \label{eq:a1-LegS}
    \end{cases}
  \end{equation} for $\alpha_{n,k}$,$\beta_{n,k}$ as defined in (\ref{eq:z1alpha}).

  Using the same arguments as in the proof of Corollary \ref{cor:HiPPO-Leg}, $\vb[n] = \parens{\frac{2n+1}{2}}^{\frac{1}{2}}$ follows from Corollary \ref{cor:approx-ftt} and (\ref{eq:leg-val-at+-2}). Also using similar arguments as the proof of Corollary \ref{cor:HiPPO-Leg}, we have
  \[\alpha_{n,k} = \begin{cases} (2n + 1)^{\frac{1}{2}}(2k + 1)^{\frac{1}{2}} \text{ if } k < n \text{ and } n - k \text{ is odd,} \\  0 \ \ \text{ is otherwise.}\end{cases}. \]

  From (8) in \cite{gu2020hippo}, we know that
  \[(z+1)P'_n(z) = nP_n(z) + (2n+1)P_{n-1}(z) + (2n-3)P_{n-2}(z)+ \dots. \]

  Including the normalization constant $(2n + 1)^{\frac{1}{2}}$, we note that $(z-1)p'_n(z) = (z+1)p'_n(z) - 2p'_n(z)$. Then we get 

  \begin{align*}
    (z+1)p'_n(z) = np_n(z) - (2n + 1)^{\frac{1}{2}}(2n - 1)^{\frac{1}{2}}p_{n-1}(z) + (2n + 1)^{\frac{1}{2}}(2n - 3)^{\frac{1}{2}}p_{n-2}(z) - \dots .
  \end{align*}
  In other words,
  \[\beta_{n,k} = \begin{cases} -(2n + 1)^{\frac{1}{2}}(2k + 1)^{\frac{1}{2}} \text{ if } k < n  \text{ and } n - k \text{ is odd,} \\
    (2n + 1)^{\frac{1}{2}}(2k + 1)^{\frac{1}{2}} \text{ if } k < n   \text{ and } n - k \text{ is even}\\
    n \ \ \text{ if } n=k \\
  0 \ \ \text{ otherwise.}\end{cases}. \]

  Recalling that the definition for $\vA_1$ from (\ref{eq:a1-LegS}), we get:

  \[\vA[n,k] = \begin{cases}
    (2n+1)^{\frac{1}{2}}\parens{2k + 1}^{\frac{1}{2}}  &\text{ if } k < n  \  \\
    n + 1 \ \ &\text{ if } k = n  \\
    0 \ \ &\text{ if } k > n \\
  \end{cases},
\]
which completes our claim.
\end{proof}

\subsection{Proof of \texorpdfstring{\cref{thm:jacobi}}{Corollary 4.1}: HiPPO for Classical Orthogonal Polynomials}
\label{sec:sllssl:jacobi}

This section proves \cref{thm:jacobi}, showing that the HiPPO matrices for measures corresponding to classical families of orthogonal polynomials \citep{chihara} are quasiseparable.
We define quasi-separability in \cref{app:qs}.
\cref{thm:hippo-jac} proves the claimed result for Jacobi polynomials and \cref{lmm:lagt-qs} proves the claimed result for Laguerre polynomials.

We note that there is a third family of classical OPs, the Hermite polynomials \cite{chihara}, which have a two-sided infinite measure.
However, since HiPPO is about continuous-time memorization of a function's \emph{history}, it requires a one-sided measure and therefore the Hermite polynomials are not appropriate.

\subsubsection{Quasiseparable Matrices}
\label{app:qs}
\begin{definition}[from \citep{EG99}]
A matrix $\vR \in \R^{N \times N}$ is $(p,q)$-quasiseparable if
\begin{itemize}
    \item Every matrix contained strictly above the diagonal has rank at most $p$.
    \item Every matrix contained strictly below the diagonal has rank at most $q$.
\end{itemize}

A $(q,q)$-quasiseparable matrix is called $q$-quasiseparable.
\label{def:quasi}
\end{definition}

We are interested in showing the $\vA$ matrices for a broad class of OPs in Corollary \ref{cor:approx-ftt} are $O(1)$-quasiseperable. We now state some properties of $q$-quasiseparable matrices:
\begin{lemma}
Let $\vQ$ be $q$-quasiseparable. Then:
\begin{enumerate}[label=(\roman*)]
 \item For any $q'$-quasiseparable matrix $\vQ' \in \R^{N \times N}$, $\vQ \pm \vQ'$ is $(q+q')$-quasiseparable.
    \item For any $\vE \in \R^{N \times N}$, $\vE$ is $r$-quasiseparable where $r =  \text{rank}(\vE)$.
    \item For any two diagonal matrices $\vD_1, \ \vD_2 \in  \R^{N \times N}$, $\vD_1\vQ\vD_2$ is $q $-quasiseparable.
\end{enumerate}
\label{lmm:quasi-prop}
\end{lemma}
\begin{proof}
We argue each point separately:

\begin{enumerate}[label=(\roman*)]
\item Any submatrix contained strictly below or above the diagonal in $\vQ$ has rank  $\leq q$ and its corresponding submatrix in $\vQ'$ also has rank $ \leq q'$. This implies that the corresponding submatrix in $\vQ \pm \vQ'$ has rank $\leq q + q'$. Therefore $\vQ \pm \vQ'$ is $(q+q')$-quasiseparable.

    \item Let the $r = \text{rank}(\vE)$. Thus any submatrix in $\vE$ has rank $\leq r$. Then $\vE$ is $r$-quasiseparable.

    \item Multiplication by diagonal matrices only scales the rows and columns, leaving the rank of each submatrix unchanged.
\end{enumerate}

\end{proof}

\subsubsection{Jacobi Polynomials}

The  Jacobi polynomial of degree $n$ with parameters $\alpha,\beta > -1$ will be denoted $\jab{n}\parens{z}$.
The Jacobi polynomials are orthogonal with respect to measure $\omega(z) = (1-z)^{\alpha}(1+z)^{\beta}$.
In particular, it is known from (eq. (4.3.3) from \cite{szego1975orthogonal})
that

 \[ \int_{-1}^1 \jab{n}\parens{z}\jab{m}\parens{z}\omega(z)\mathrm{d}z = \frac{2^{\alpha+\beta+1}}{2n+\alpha + \beta + 1}\cdot \frac{\Gamma(n+\alpha+1)\Gamma(n+\beta+1)}{\Gamma(n+\alpha+\beta+1)n!} \delta_{n,m},\]

where $\Gamma(\cdot)$ is the gamma function.
Let \[\lambda^{\alpha,\beta}_n = \parens{\frac{2^{\alpha+\beta+1}}{2n+\alpha + \beta + 1}\cdot \frac{\Gamma(n+\alpha+1)\Gamma(n+\beta+1)}{\Gamma(n+\alpha+\beta+1)n!}}^{\frac{1}{2}}\] be our normalization constant.
We note that the normalized Jacobi polynomials
\begin{equation}
    \pab{n}\parens{z} = \frac{\jab{n}\parens{z}}{\lambda^{\alpha,\beta}_n}
    \label{eq:pab}
\end{equation} form an orthonormal OP family.

We now discuss some useful properties of Jacobi polynomials. It is known that (\cite{Shen2011}, eq. (3.100)):

\begin{equation}
    \jab{n}\parens{z}=\frac{1}{n + \alpha + \beta} \left[(n+\beta)\jabpo{n}{-1}(z)+(n+\alpha)\japob{n}{-1}\parens{z}\right].
    \label{eq:jab}
\end{equation}

From (4.21.7) in \cite{szego1975orthogonal},  it is known that the derivative of $\jab{n}\parens{z}$ is proportional to  $\jabp{n-1}{+1}\parens{z}$:

\begin{equation}
    \frac{\partial}{\partial z}\jab{n}(z)=\frac{1}{2} (n + \alpha + \beta + 1)\jabp{n-1}{+1}\parens{z}.
    \label{eq:dzjab}
\end{equation}

 From (\ref{eq:jab}) and (\ref{eq:dzjab}), it follows that

\begin{align}
    \frac{\partial}{\partial z}\jab{n}\parens{z}&=\frac{1}{2} \cdot \parens{(n+\beta) \japob{n-1}{+1}\parens{z}+(n+\alpha)\jabpo{n-1}{+1}\parens{z} }. \label{eq:pnab}
\end{align}

Additionally, the Jacobi polynomials $\japob{n-1}{+1}\parens{z}$ and $\jabpo{n-1}{+1}\parens{z}$ can be written as sums of $\jab{n-1}\parens{z}$ polynomials. In particular from \cite{Shen2011}  (3.112) and (3.115),

\begin{equation}
    \japob{n-1}{+1}\parens{z}=\frac{\Gamma(n + \beta )}{\Gamma(n + \alpha + \beta + 1)} \cdot \sum_{k=0}^{n-1} \frac{(2k+ \alpha + \beta + 1)\Gamma(k+ \alpha + \beta + 1)}{\Gamma(k + \beta+ 1)} \jab{k}\parens{z},
    \label{eq:japo-b}
\end{equation}
 and

\begin{equation}
    \jabpo{n-1}{+1}\parens{z}=\frac{\Gamma(n + \alpha )}{\Gamma(n + \alpha + \beta + 1)}  \cdot \sum_{k=0}^{n-1} (-1)^{n-k-1}\frac{(2k+ \alpha + \beta + 1)\Gamma(k+ \alpha + \beta + 1)}{\Gamma(k + \alpha+ 1)} \jab{k}\parens{z}.
    \label{eq:ja-bpo}
\end{equation}

Using (\ref{eq:japo-b}) and (\ref{eq:ja-bpo}) in (\ref{eq:pnab}) allows us to write $\frac{\partial}{\partial z}\jab{n}\parens{z}$ as a sum of $\set{\jab{k}\parens{z}}_{k \leq n}$ as follows:

  \begin{align}
   &\frac{\partial}{\partial z}\jab{n}\parens{z}  =\frac{n + \beta}{2}  \parens{\frac{\Gamma(n + \beta)}{\Gamma(n + \alpha + \beta + 1)}  \sum_{k=0}^{n-1} \frac{(2k+ \alpha + \beta + 1)\Gamma(k+ \alpha + \beta + 1)}{\Gamma(k + \beta+ 1)} \jab{k}\parens{z}}  \nonumber \\ \ \ \ \
     &-\frac{n + \alpha}{2}  \parens{\frac{\Gamma(n + \alpha)}{\Gamma(n + \alpha + \beta + 1)}   \sum_{k=0}^{n-1} (-1)^{n-k}\frac{(2k+ \alpha + \beta + 1)\Gamma(k+ \alpha + \beta + 1)}{\Gamma(k + \alpha+ 1)} \jab{k}\parens{z}}. \label{eq:jnab-sum}
\end{align}

We use these properties to write $\frac{\partial}{\partial z}\pab{n}\parens{z}$ as a sum of $\set{\pab{k}\parens{z}}_{k \leq n}$:

\begin{corollary}
Let $\pab{n}\parens{z}$ and $\lambda^{\alpha,\beta}_{n}$ be as defined in (\ref{eq:pab}).

Then
  \begin{align*}
   \frac{\partial}{\partial z}&\lambda^{\alpha,\beta}_{n}\pab{n}\parens{z}  =\frac{(n + \beta)}{2} \cdot \\ &\parens{\frac{\Gamma(n + \beta)}{\Gamma(n + \alpha + \beta + 1)}  \sum_{k=0}^{n-1} \frac{(2k+ \alpha + \beta + 1)\Gamma(k+ \alpha + \beta + 1)}{\Gamma(k + \beta+ 1)} \lambda^{\alpha,\beta}_k\pab{k}\parens{z}}  \nonumber \\ &\ \ \ \ -
     \frac{(n + \alpha)}{2} \cdot \\ &\parens{\frac{\Gamma(n + \alpha )}{\Gamma(n + \alpha + \beta + 1)}  \sum_{k=0}^{n-1} (-1)^{n-k}\frac{(2k+ \alpha + \beta + 1)\Gamma(k+ \alpha + \beta + 1)}{\Gamma(k + \beta+ 1)} \lambda^{\alpha,\beta}_k\pab{k}\parens{z}}
\end{align*}

 \label{cor:pnab-sum}
\end{corollary}
\begin{proof}

   Recall that $\jab{n}\parens{z} = \lambda^{\alpha,\beta}_{n}\pab{n}$. Then the claim follows from $(\ref{eq:jnab-sum})$.

\end{proof}

\subsubsection{HiPPO for Jacobi Polynomials}

\begin{theorem}
Let $\pab{n}\parens{z}$ be defined as in (\ref{eq:pab}) and $\omega(z) = (1-z)^{\alpha}(1+z)^{\beta}$. Then

\begin{align*}
    \dot x_n(t) \approx -\frac{1}{\theta}\vA x(t) + \frac{2}{\theta}\vb u(t)
\end{align*}

where $\vA$ is $3$-quasiseperable.
\label{thm:hippo-jac}
\end{theorem}
\begin{proof}
  From Corollary \ref{cor:approx-ftt},
  \[\dot x_n(t) \approx - \frac{1}{\theta}\vA x(t) + \frac{2}{\theta}u(t)\begin{bmatrix} \vdots \\ \pab{n}(1) \\ \vdots \end{bmatrix}\]

  where $\vA = \vA_1 + 2\vA_2$ for  $\vA_1$ as defined in Corollary \ref{cor:thetat-theta} and $\vA_2[n,k] =  \pab{n}(-1)\pab{n}(-1)$. 

  From Corollary \ref{cor:pnab-sum}, we observe that

  \begin{equation}
    \vA_1[n,k] = 2 \cdot \begin{cases} \  \frac{(n + \beta)}{2 \ \lambda^{\alpha,\beta}_{n}} \cdot  \frac{\Gamma(n + \beta)}{\Gamma(n + \alpha + \beta + 1)}\cdot \parens{ \frac{(2k+ \alpha + \beta + 1)\Gamma(k+ \alpha + \beta + 1)}{\Gamma(k + \beta+ 1)} \lambda^{\alpha,\beta}_k } - \\ \ \ \ \ \
      \frac{(n + \alpha)}{2 \ \lambda^{\alpha,\beta}_{n}} \cdot  \frac{\Gamma(n + \alpha )}{\Gamma(n + \alpha + \beta + 1)}  \cdot\parens{ (-1)^{n-k}\frac{(2k+ \alpha + \beta + 1)\Gamma(k+ \alpha + \beta + 1)}{\Gamma(k + \alpha+ 1)} \lambda^{\alpha,\beta}_k }
      &\text{ if } k < n 

    \\ \  0 \ \ &\text{ otherwise}\end{cases}.
    \label{eq:a1nk}
  \end{equation}

  Then we note that,
  \begin{equation}
    \vA_1 =\vD_{11}\vQ_{1}\vD_{12} - \vD_{21}\vQ_{1}\vD_{22},
    \label{eq:a1-jac}
  \end{equation} where $\vD_{11}, \vD_{12},\vD_{21}, \vD_{22}$ are the diagonal matrices such that 

  \[\vD_{11}[n,n] = \frac{1}{\lambda^{\alpha,\beta}_{n}} \cdot \frac{\Gamma(n + \beta+1 )}{\Gamma(n + \alpha + \beta + 1)}, \]

  \[\vD_{12}[k,k] =  \frac{(2k+ \alpha + \beta + 1)\Gamma(k+ \alpha + \beta + 1)}{\Gamma(k + \beta+ 1)} \lambda^{\alpha,\beta}_k,\]

  \[\vD_{21}[n,n] = (-1)^{n}\cdot\frac{(1}{\lambda^{\alpha,\beta}_{n}} \cdot \frac{\Gamma(n + \alpha +1)}{\Gamma(n + \alpha + \beta + 1)} \]

  \[  \vD_{22}[k,k] =  (-1)^{k} \cdot \frac{(2k+ \alpha + \beta + 1)\Gamma(k+ \alpha + \beta + 1)}{\Gamma(k + \alpha+ 1)} \lambda^{\alpha,\beta}_k,\]
  and

  \[\vQ_1[n,k] = \begin{cases}
    1  &\text{ if } k < n  \\  0 \ \ &\text{ otherwise.}
  \end{cases}.\]

  (\ref{eq:a1-jac}) makes use of the fact that $(-1)^{n+k} = (-1)^{n-k}$ along with the definitions above. 

  Any submatrix of $\vQ_1$ below the diagonal contains all $1$s, and submatrix of $\vQ_1$ above the diagonal contains all $0$s. Then any submatrix above or below the diagonal has rank 1. Therefore $\vQ_1$ is 1-quasiseparable. Since $\vQ_1$ is 1-quasiseparable and $\vD_{11}, \vD_{12},\vD_{21}, \vD_{22}$ are all diagonal matrices, part (iii) of Lemma \ref{lmm:quasi-prop} implies that the matrices $\vD_{11}\vQ_{1}\vD_{12}$ and  $\vD_{21}\vQ_{1}\vD_{22}$ are both 1-quasiseparable. Therefore  part (i) of Lemma \ref{lmm:quasi-prop} implies that $\vA_1$ is 2-quasiseparable.

  From (4.1.1) and (4.1.4) in \cite{szego1975orthogonal}, it is known that 
\[\pab{n}(1) = \frac{1}{\lambda^{\alpha,\beta}_{n}}\binom{n+\alpha}{n} \text{ and } \pab{n}(1) = \frac{(-1)^n}{\lambda^{\alpha,\beta}_{n}}\binom{n+\beta}{n}\]

  where
\[\binom{z}{n} = \begin{cases} \ 
    \frac{\Gamma(z + 1 )}{\Gamma(n + 1)\Gamma(z - n + 1)} \   &\text{ if } n \geq 0 \\ \ 
    0 \ \ \  &\text{ if } n < 0 
  \end{cases}.\]
Then $\vA_2$ can be written $\vD_3\vQ_2\vD_4$ where $\vD_3, \vD_4$ are the diagonal matrices such that \[\vD_3[n,n] = \frac{(-1)^n}{\lambda^{\alpha,\beta}_{n}}\binom{n+\beta}{n} \text{, } \vD_4[k,k] = \frac{(-1)^k}{\lambda^{\alpha,\beta}_{k}}\binom{k+\beta}{k},\]

  where $\vQ_2[n,k] = 1$ for all $ 0 \leq n,k <N$. $\vQ_2$ has rank 1, and $\vD_3, \vD_4$ are diagonal matrices. Hence by part (ii) and (iii) Lemma \ref{lmm:quasi-prop}, $\vA_2$ is 1-quasiseparable.

  Since $\vA_1$ is 2-quasiseparable and $\vA_2$ is  1-quasiseparable, part (i) of Lemma \ref{lmm:quasi-prop} implies that $\vA = \vA_1 + 2\vA_2$ is  3-quasiseparable and the claim follows.
\end{proof}

\subsubsection{HiPPO-LagT}
The  Laguerre polynomial of degree $n$ with parameters $\alpha > -1$ will be denoted $L^{\alpha}_{n}\parens{z}$. The Laguerre polynomials are orthogonal with respect to measure $z^{\alpha}e^{-z}$. In particular, from (5.1.1) in \cite{szego1975orthogonal} we know
that

 \[ \int_{-1}^\infty L^{\alpha}_{n}\parens{z}L^{\alpha}_{m}\parens{z}z^{\alpha}e^{-z}\mathrm{d}z = \frac{\Gamma(n + \alpha + 1)!}{\Gamma(n+1)}\delta_{n,m}.\]

Let $\lambda_n = \parens{ \frac{\Gamma(n+1)}{\Gamma(n+\alpha+1)}}^{\frac{1}{2}}$ be our normalization constant. We note that the normalized Laguerre polynomials \begin{equation}
    p_{n}\parens{z} = \lambda_n L^{\alpha}_n(t-\x)
    \label{eq:lag}
\end{equation} form an orthonormal OP family with respect to measure $\omega = (t-\x)^{\alpha}e^{-(t-\x)}\mathbbm{1}_{(-\infty,t)}$ for a fixed $\alpha$ and tilting $\chi = (t-\x)^{\alpha}\exp\parens{-\frac{1-\beta}{2}(t-\x)}\mathbbm{1}_{(-\infty,t)}$ for a fixed $\beta$.

We use the following result from \cite{gu2020hippo}:

\begin{theorem}
Let $p_n(z)$ be defined as in $(\ref{eq:lag})$. Then

\begin{align*}
     \dot x_n(t) = -\vA x(t) + \vb u(t)
\end{align*}

where \[\vA[n,k] = \begin{cases}
  \frac{1+\beta}{2} \ \ &\text{  if }  \ \ k = n \\
  1  &\text{ if } \ \  k < n  \\
 0 \ \ &\text{ otherwise} \\
\end{cases},\] \[\vb[n] = \lambda_n\binom{n + \alpha}{n},\] 

\label{thm:lagt-hippo}
\end{theorem}

We now show that $\vA$ as defined in Theorem \ref{thm:lagt-hippo} is $1$-quasiseperable.

\begin{lemma}
\label{lmm:lagt-qs}
Let $\vA$ be defined as in Theorem \ref{thm:lagt-hippo}. Then $\vA$ is $1$-quasiseperable.
\end{lemma}
\begin{proof}
From Theorem \ref{thm:lagt-hippo}, we know that

\[\vA[n,k] = \begin{cases}
  \frac{1+\beta}{2} \ \ &\text{  if }  \ \ k = n \\
  1  &\text{ if } \ \  k < n  \\
 0 \ \ &\text{ otherwise} \\
\end{cases},\] \[\vb[n] = \lambda_n\binom{n + \alpha}{n}.\]

Below the diagonal, all entries $\vA[n,k] = 1$. Then any submatrix below the diagonal has rank 1. Similarly, above the diagonal, all entries $\vA[n,k] = 0$. Then any submatrix above the diagonal also has rank 1. Then by Definition \ref{def:quasi}, the claim follows.

\end{proof}

\section{LSSL Algorithms}
\label{sec:sllssl-algorithms}

\begin{itemize}%
  \item \cref{sec:sllssl:krylov} proves \cref{thm:krylov}, providing an algorithm to compute the Krylov function efficiently for LSSLs.
  \item \cref{sec:sllssl:trid} shows a further simplification of \cref{thm:jacobi}, presenting an even simpler class of structured matrices that we use in our implementation of LSSL.
  \item \cref{sec:sllssl:implementation} provides technical details of the implementation of LSSL, in particular for computing the MVM black box (multiplication by \( \overline{A} \)) and for computing gradients during backpropagation.
\end{itemize}

\subsection{Proof of \texorpdfstring{\cref{thm:krylov}}{Theorem 2}}
\label{sec:sllssl:krylov}

This section addresses the computational aspects of the LSSL.
In particular, we prove \cref{thm:krylov} for the computational speed of computing the Krylov function \eqref{eq:krylov} for quasiseparable matrices \( A \),
by providing a concrete algorithm in \cref{sec:krylov:alg}.

We restate the Krylov function \eqref{eq:krylov} here for convenience. Recall that \( L \) is the length of the input sequence and \( N \) is the order of the LSSL internate state, e.g. \( A \in \mathbb{R}^{N \times N} \).
\begin{align*}
  \mathcal{K}_L(A, B, C) = \left(C A^i B\right)_{i \in [L]} \in \mathbb{R}^L = (CB, CAB, \dots, CA^{L-1}B)
\end{align*}

\begin{remark}%
  We call \eqref{eq:krylov} the \emph{Krylov function} following the notation of \citep{de2018two}, since it can be written \( \mathcal{K}(A, B)^T C \) where \( \mathcal{K}(A, B) \) is the Krylov matrix defined in \eqref{eq:krylov-matrix}.
  Alternative naming suggestions are welcome.
\end{remark}

\subsubsection{The Algorithm}
\label{sec:krylov:alg}

We follow the similar problem of \citep[Lemma 6.6]{de2018two} but track the dependence on \( L \) and the log factors more precisely, and optimize it in the case of stronger structure than quasiseparability, which holds in our setting (particularly \cref{thm:trid}).

The first step is to observe that the Krylov function \( \mathcal{K}_L(A, B, C) \) is actually the coefficient vector of \( C(I - Ax)^{-1}B \pmod{x^L} \) as a polynomial in \( x \).
(Note that \( Ax \) means simply multiplying every entry in \( A \) by a scalar variable \( x \).)
This follows from expanding the power series \( (I - Ax)^{-1} = I + Ax + A^2 x^2 + \dots \).
Thus we first compute \( C(I - Ax)^{-1} B \), which is a rational function of degree at most \( N \) in the numerator and denominator (which can be seen by the standard adjoint formula for the matrix inverse).

The second step is simply inverting the denominator of this rational function \( \pmod{x^L} \) and multiplying by the numerator, both of which are operations that need \( L \log(L) \) time by standard results for polynomial arithmetic \citep{shoup2009computational}.

For the remainder of this section, we focus on computing the first part.
We make two notational changes:
First, we transpose \( \vC \) to make it have the same shape as \( \vB \).
We consider the more general setting where \( \vB \) and \( \vC \) have multiple columns; this can be viewed as handling a ``batch'' problem with several queries for \( \vB, \vC \) at the same time.
\begin{lemma}
  \label{lmm:resolvent}
  Let $\vA$ be a $q$-quasiseparable matrix.
  Then
  \begin{align*}
    \vC^T (\vI - \vA x)^{-1} \vB \qquad \text{where } \vA \in \mathbb{R}^{N \times N} \qquad \vB, \vC \in \mathbb{R}^{N \times k}
  \end{align*}
  is a \( k \times k \) matrix of rational functions of degree at most \( N \),
  which can be computed in \( O(q^3 \log^4 N )  \) operations.
\end{lemma}

The main idea is that quasiseparable matrices are recursively ``self-similar'', in that the principal submatrices are also quasiseparable, which leads to a divide-and-conquer algorithm.
In particular, divide \( \vA = \begin{bmatrix} \vA_{00} & \vA_{01} \\ \vA_{10} & \vA_{11} \end{bmatrix} \) into quadrants.
Then by \cref{def:quasi}, \( \vA_{00}, \vA_{11} \) are both \( q \)-quasiseparable and \( \vA_{01}, \vA_{10} \) are rank \( q \).
Therefore the strategy is to view \( \vI - \vA x \) as a low-rank perturbation of smaller quasiseparable matrices and reduce the problem to a simpler one.

\begin{proposition}[Binomial Inverse Theorem or Woodbury matrix identity~\cite{woodbury1950,golub2013matrix}]
\label{prop:woodbury}
Over a commutative ring $\mathcal{R}$, let $\vA \in \mathcal{R}^{N \times N}$ and $\vU,\vV \in \mathcal{R}^{N \times p}$. Suppose $\vA$ and $\vA+\vU\vV^T$ are invertible. Then $\vI_p + \vV^T\vA^{-1}\vU \in \mathcal{R}^{p \times p}$ is invertible and
  \[
    (\vA + \vU\vV^T)^{-1} = \vA^{-1} - \vA^{-1}\vU(\vI_p + \vV^T\vA^{-1}\vU)^{-1}\vV^T\vA^{-1}
  \]
\end{proposition}

For our purposes, $\mathcal{R}$ will be the ring of rational functions over $\R$.

\begin{proof}[Proof of \cref{lmm:resolvent}]

  Since $\vA$ is $q$-quasiseparable, we can write $\vA_{10} = \vU_L \vV_L^T$ and $\vA_{01} = \vU_U \vV_U^T$ where $\vU_\cdot, \vV_\cdot \in \F^{N \times q}$.
  Notice that we can write $\vI-\vA x$ as
  \begin{align*}
    \vI-\vA x =
    \begin{bmatrix} \vI-\vA_{00} x & \vzero \\ \vzero & \vI-\vA_{11} x \end{bmatrix}
    +
    \begin{bmatrix} \vzero & \vU_U \\ \vU_L & \vzero \end{bmatrix}
    \begin{bmatrix} \vV_L & \vzero \\ \vzero & \vV_U \end{bmatrix}^T
    x
    .
  \end{align*}

  Suppose we know the expansions of each of
  \begin{align}
    & \vM_1 \in \mathcal{R}^{k \times k} = \vC^T \begin{bmatrix} \vI-\vA_{00} x & \vzero \\ \vzero & \vI-\vA_{11} x \end{bmatrix}^{-1} \vB \\
    & \vM_2 \in \mathcal{R}^{k \times 2q} = \vC^T
    \begin{bmatrix} \vI-\vA_{00} x & \vzero \\ \vzero & \vI-\vA_{11} x \end{bmatrix}^{-1}
    \begin{bmatrix} \vzero & \vU_U \\ \vU_L & \vzero \end{bmatrix} \\
    & \vM_3 \in \mathcal{R}^{2q \times 2q} = \begin{bmatrix} \vV_L & \vzero \\ \vzero & \vV_U \end{bmatrix}^T
    \begin{bmatrix} \vI-\vA_{00} x & \vzero \\ \vzero & \vI-\vA_{11} x \end{bmatrix}^{-1}
    \begin{bmatrix} \vzero & \vU_U \\ \vU_L & \vzero \end{bmatrix} \\
    & \vM_4 \in \mathcal{R}^{2q \times k} = \begin{bmatrix} \vV_L & \vzero \\ \vzero & \vV_U \end{bmatrix}^T
    \begin{bmatrix} \vI-\vA_{00} x & \vzero \\ \vzero & \vI-\vA_{11} x \end{bmatrix}^{-1}
    \vB.
  \end{align}

  By \cref{prop:woodbury}, the desired answer is
  \[
    \vC^T (X-\vA)^{-1} \vB = \vM_1 - \vM_2(\vI_{2q}+\vM_3)^{-1}\vM_4.
  \]
  Then the final result can be computed by inverting $\vI_{2t}+\vM_3$ ($O(q^3 N\log(N))$ operations),
  multiplying by $\vM_2,\vM_4$ ($O((kq^2 + k^2q) N\log(N))$ operations),
  and subtracting from $\vM_1$ ($O(k^2 N\log(N))$ operations).
  This is a total of $O( (q^3 + kq^2 + k^2q) N\log(N) )$ operations.
  Note that when $k = O(q\log N)$, this becomes $O(q^3 N \log^3{N})$; we will use this in the analysis shortly.

  To compute $\vM_1,\vM_2,\vM_3,\vM_4$, it suffices to compute the following:
  \begin{equation}
    \label{eq:krylov:recurse1}
    \begin{aligned}%
      \vC_1^T (\vI-\vA_{00}x)^{-1}\vB_0 \qquad& \vC_1^T (\vI-\vA_{11}x)^{-1}\vB_1 \\
      \vC_0^T (\vI-\vA_{00}x)^{-1}\vU_U  \qquad&\vC_1^T (\vI-\vA_{11}x)^{-1}\vU_L \\
      \vV_L^T (\vI-\vA_{00}x)^{-1}\vU_U \qquad& \vV_U^T (\vI-\vA_{11}x)^{-1}\vU_L \\
      \vV_L^T (\vI-\vA_{00}x)^{-1}\vB_0 \qquad& \vV_U^T (\vI-\vA_{11}x)^{-1}\vB_1.
    \end{aligned}
  \end{equation}
  But to compute those, it suffices to compute the following $(k+t) \times (k+t)$ matrices:
  \begin{equation}
    \label{eq:krylov:recurse2}
    \begin{aligned}
      & \begin{bmatrix} \vC_0 & \vV_L \end{bmatrix}^T
      (\vI-\vA_{00}x)^{-1}
      \begin{bmatrix} \vB_0 & \vU_U \end{bmatrix} \\
      & \begin{bmatrix} \vC_1 & \vV_U \end{bmatrix}^T
      (\vI-\vA_{11}x)^{-1}
      \begin{bmatrix} \vB_1 & \vU_L \end{bmatrix} \\
    \end{aligned}
  \end{equation}

  Since $\vA_{00}$ and $\vA_{11}$ have the same form as $\vA$, this is two recursive calls of half the size.
  Notice that the size of the other input (dimensions of $\vB,\vC$) is growing, but when the initial input is $k=1$, it never exceeds $1+q\log{N}$ (since they increase by $q$ every time we go down a level).
  Earlier, we noticed that when $k = O(q\log N)$, the reduction step has complexity $O(q^3 N \log^3(N))$ for any recursive call.
  The recursion adds an additional $\log{N}$ multiplicative factor on top of this.
\end{proof}

\begin{corollary}%
  \label{cor:resolvent}
  Suppose that \( \vA \) is semiseparable instead of quasiseparable,
  and suppose \( q \) is a small constant.
  Then the cost of \cref{lmm:resolvent} is \( O(N \log^2(N)) \) operations.
\end{corollary}
This follows from the fact that in the recursion \eqref{eq:krylov:recurse1} and \eqref{eq:krylov:recurse2}, the \( \vU, \vV \) matrices do not have to be appended if they already exist in \( \vB, \vC \).
For intuition, this happens in the case when \( \vA \) is tridiagonal, so that \( U, V \) have the structure \( (1, 0, \dots, 0) \), or the case when the off-diagonal part of \( \vA \) is all \( 1 \) (such as the HiPPO-LegT matrix).
The matrices in \cref{sec:sllssl:jacobi} and \cref{sec:sllssl:trid} (\cref{thm:hippo-jac,thm:lagt-hippo,thm:trid}) actually satisfy this stronger structure, so \cref{cor:resolvent} applies.

Combining everything, this proves \cref{thm:krylov} with the exact bound \( N \log^2(N) + L \log(L) \) operations.
The memory claim follows similarly, and the depth of the algorithm is \( \log^2(N) + \log(L) \) from the divide-and-conquer recursions.

\subsubsection{Summary of Computation Speed for LSSLs and other Mechanisms}

We provide a summary of complexity requirements for various sequence model mechanisms, including several versions of the LSSL.
Note that these are over exact arithmetic as in \cref{thm:krylov}.

First, the self-attention mechanism is another common sequence model that has an \( L^2 \) dependence on the length of the sequence, so it is not suitable for the very long sequences we consider here.
(We do note that there is an active line of work on reducing this complexity.)

Second, we include additional variants of the LSSL.
In \cref{tab:complexity-full}, LSSL-naive denotes learning \( A \) and \( \dt \) for unstructured \( A \);
LSSL-fixed denotes not learning \( A, \dt \) (see \cref{sec:model} for details);
LSSL denotes the learning \( A \) and \( \dt \) for the structured class \( A \).

\begin{table}
  \caption{
    Complexity of various sequence models in terms of length (\textbf{L}), batch size (\textbf{B}), and hidden dimension (\textbf{H}).
    Measures are parameter count, training computation, memory requirement, and inference computation for 1 sample and time-step.
  }
  \centering
  \begin{tabular}{@{}lll@{}}
    \toprule
               & Convolution                      & RNN                                  \\
    \midrule
    Parameters & \( LH^2 \)                       & \( H^2 \)                            \\
    Training   & \( BLH^2 + L\log(L)(H^2 + BH) \) & \( BLH^2 \)                          \\
    Memory     & \( BLH + LH^2 \)                 & \( BLH \)                            \\
    Parallel   & Yes                              & No                                   \\
    Inference  & \( LH^2 \)                       & \( H^2 \)                            \\
    \midrule
               & Attention                        & LSSL-naive                                 \\ \midrule
    Parameters & \( H^2 \)                        & \( HN + N^2 \)                       \\
    Training   & \( B(L^2H + LH^2) \)             & \( HN^3 + LHN^2 + BL\log(L) H N \)   \\
    Memory     & \( B(L^2 + HL) \)                & \( HN^2 + LHN + BLH \)               \\
    Parallel   & Yes                              & Yes                                  \\
    Inference  & \( L^2H + H^2L \)                & \( HN^2 \)                           \\
    \midrule
               & LSSL-fixed     & LSSL                               \\
    \midrule
    Parameters & \( HN \)                         & \( HN \)                             \\
    Training   & \( BL\log(L) H N \)              & \( BH(N\log^2N + L\log L) + BL\log(L) H \) \\
    Memory     & \( LHN + BLH \)                  & \( BHL \)                            \\
    Parallel   & Yes                              & Yes                                  \\
    Inference  & \( HN^2 \)                       & \( HN \)                             \\
    \bottomrule
  \end{tabular}
  \label{tab:complexity-full}
\end{table}

We include brief explanations of these complexities for the LSSL variants.

\paragraph{LSSL-naive}

\begin{itemize}%
  \item Parameters: \( O(HN) \) in the matrices \( B, C \) and \( O(N^2) \) in the matrix \( A \).
  \item Training: \( O(H N^3) \) to invert compute the matrix \( \overline{A} \) for all \( H \) features. \( O(LHN^2) \) to compute the Krylov matrix \( C, CA, \dots \).
    \( O(B L\log(L) H N \) to multiply by \( B \) and convolve with \( u \).
  \item Memory: \( O(HN^2) \) to store \( \overline{A} \). \( O(LHN) \) to store the Krylov matrix. \( O(BLH) \) to store the inputs/outputs
  \item Inference: \( O(HN^2) \) to for MVM by \( \overline{A} \).
\end{itemize}

\paragraph{LSSL-fixed}

\begin{itemize}
  \item Parameters: \( O(HN) \) in the matrices \( C \).
  \item Training: \( O(B L\log(L) H) \) to convolve with \( u \).
  \item Memory: \( O(LHN) \) to store the Krylov matrix (but cached, so no backprop). \( O(BLH) \) for inputs/outputs.
  \item Inference: \( O(HN^2) \) to for MVM by \( \overline{A} \).
\end{itemize}

\paragraph{LSSL}

\begin{itemize}
  \item Parameters: \( O(HN) \) for \( A, B, C, \dt \).
  \item Training: \( B H \cdot \tilde{O}(N + L) \) to compute Krylov, \( O(B L\log(L) H) \) for the convolution.
  \item Memory: \( O(BHL) \) to store Krylov (and inputs/outputs).
  \item Inference: \( O(HN) \) to multiply \( x_t [H, N] \) by \( \overline{A} [H, N, N] \)
\end{itemize}

\subsection{Further Simplification with Tridiagonal Matrices}
\label{sec:sllssl:trid}

The algorithm for \cref{thm:krylov} for general quasiseparable matrices is still difficult to implement in practice, and we make a further simplification using a particular subclass of quasiseparable matrices.
\begin{theorem}%
  \label{thm:trid}
  The class of \( N \times N \) matrices \( \mathcal{S}_N = \{P (D + T^{-1}) Q \} \) with diagonal \( D, P, Q \) and tridiagonal \( T \)
  includes the original HiPPO-LegS, HiPPO-LegT, and HiPPO-LagT matrices~\citep{gu2020hippo}.
\end{theorem}
\cref{thm:trid} shows that a simple representation involving tridiagonal and diagonal matrices captures all of the original HiPPO matrices.
In particular, our LSSL implementation initializes \( A \) to be the HiPPO-LegS matrix (\cref{sec:model}) and learns within the class defined by \cref{thm:trid}.

We note that the matrices in \cref{thm:trid} are all 1-quasiseparable and in particular also contain the HiPPO matrices for Gegenbauer and generalized Laguerre orthogonal polynomials derived in \cref{thm:lagt-hippo}.
In fact, the notion of semiseparability, which is closely related to (and actually is the predecessor of) quasiseparability,
was originally motivated precisely to capture inverses of tridiagonal matrices.
Thus the structured class in \cref{thm:trid} can be viewed as an approximation of \( 3 \)-quasiseparable matrices (\cref{thm:jacobi}) to \( 1 \)-quasiseparable,
which still contains many of the HiPPO families of interest.

\begin{proof}
  We simply show that each of these specific matrices can be represented in the proposed form.

  \textbf{HiPPO-LegT.}

  Let \( A \) denote the HiPPO-LegT transition matrix.
  Up to row/column scaling (i.e.\ left- and right- multiplication by diagonal \( P \) and \( Q \)),
  we can write
  \begin{align*}
    A_{nk} =
    \begin{cases}
      (-1)^{n-k} & \mbox{if } n \le k \\
      1 & \mbox{if } n \ge k
    \end{cases}
    .
  \end{align*}

  The main observation is that
  \begin{align*}
    A^{-1} = 
    \frac{1}{2}
    \begin{bmatrix}
      1 & 1 & 0 & \dots & 0 & 0 & 0 \\
      -1 & 0 & 1 & \dots & 0 & 0 & 0 \\
      0 & -1 & 0 & \dots & 0 & 0 & 0 \\
      \vdots & \vdots & \vdots & \ddots & \vdots & \vdots & \vdots \\
      0 & 0 & 0 & \dots & 0 & 1 & 0 \\
      0 & 0 & 0 & \dots & -1 & 0 & 1 \\
      0 & 0 & 0 & \dots & 0 & -1 & 1 \\
    \end{bmatrix}
  \end{align*}

  \paragraph{HiPPO-LegS.}
  The HiPPO-LegS matrix is
  \begin{align*}
    A_{nk}
    &=
    -
    \begin{cases}
      (2n+1)^{1/2}(2k+1)^{1/2} & \mbox{if } n > k \\
      n+1 & \mbox{if } n = k \\
      0 & \mbox{if } n < k
    \end{cases}
    .
  \end{align*}
  This can be written as \( -P A' Q \) where \( P = Q = \diag( (2n+1)^\frac{1}{2} ) \) and
  \begin{align*}
    A'_{nk}
    =
    \begin{cases}%
      1 & \mbox{if } n > k \\
      0 & \mbox{if } n < k \\
      1 - \frac{n}{2n+1} & \mbox{if } n = k
    \end{cases}
    .
  \end{align*}
  Finally, \( A' = D + T^{-1} \) where \( D = -\diag(\frac{n}{2n+1}) \) and \( T \) is the matrix with \( 1 \) on the main diagonal and \( -1 \) on the subdiagonal.

  \paragraph{HiPPO-LagT.}
  The HiPPO-LagT matrix is
  \begin{align*}
    A_{nk}
    &=
    -
    \begin{cases}
      1 & \mbox{if } n > k \\
      0 & \mbox{if } n < k \\
      \frac{1}{2} & \mbox{if } n = k \\
    \end{cases}
    .
  \end{align*}
  This can be written as \( -P (D + T^{-1}) Q \) where \( P = Q = I \), \( D = -\frac{1}{2}I \), and \( T \) is the same tridiagonal matrix as in the HiPPO-LegS case.
\end{proof}

\subsection{Implementation Details}
\label{sec:sllssl:implementation}

In this section we provide several implementation details that are useful for implementing LSSLs in practice.

Recall that one of the main primitives of LSSLs is the matrix-vector multiplication \( y = \overline{A} x \) (\cref{sec:lssl-interpretation}, \cref{sec:model}), where \( \overline{A} \) is the state matrix \( A \) discretized with step size \( \dt \) using the bilinear method (\cref{sec:lssl:gbt}).
In \cref{sec:mvm:bilinear}, we describe how this MVM can be performed with simpler MVM primitives which we call the ``forward difference'' and ``backward difference''.

However, if these MVM primitives are implemented in a specialized way for particular classes of \( A \) matrices (i.e., not using atoms in a standard autograd framework), then we also need to calculate several additional gradients by hand.
\cref{sec:mvm:gradients} shows that calculating gradients to \( A, \dt, x \) during backpropagation can actually be reduced to those same forward/backward difference primitives.

Finally, in the case when \( A \) is the structured class of matrices in \cref{thm:trid}, \cref{sec:sllssl:trid} shows how to efficiently calculate those primitives using a black-box tridiagonal solver.
Our code\footnote{Available at \url{https://github.com/HazyResearch/state-spaces}} implements all the algorithms in this section, with bindings to the cuSPARSE library for efficient tridiagonal solving on GPU.

\subsubsection{Matrix-vector Multiplication by the Bilinear Discretization}
\label{sec:mvm:bilinear}

\paragraph{Bilinear Discretization}
The discrete state-space system is given by \eqref{eq:1-discrete} and \eqref{eq:2-discrete}, re-written here for convenience
\begin{align*}
  x_{t} &= \overline{A} x_{t-1} + \overline{B} u_t \\
  y_t &= C x_t + D u_t
\end{align*}
where \( \overline{A} \) is a function of \( A, \delta_t \) and \( \overline{A} \) is a function of \( A, B, \delta_t \).
In particular, we define \( \overline{A} \) to be the matrix discretized using the bilinear method (\cref{sec:lssl:gbt}),
and the system can be written explicitly:
\begin{align*}
  x_{t} &= \left( I - \frac{\dt A}{2} \right)^{-1} \left( \left( I + \frac{\dt A}{2} \right) x_{t-1}+ \dt B u_t \right) \\
  y_t &= C x_t + D u_t
\end{align*}

Thus it suffices to compute the maps
\begin{align*}
  F(A, \dt, x) := (I + \dt A) x
\end{align*}
and
\begin{align*}
  B(A, \dt, x) := (I + \dt A)^{-1} x
  .
\end{align*}
We will call these functions the \emph{forward difference} and \emph{backward difference} maps, respectively.
(The Euler and backward Euler discretizations (\cref{sec:lssl:euler}) are also known as the ``forward difference'' and ``backward difference'' methods, which in the case of linear systems reduces down to the maps \( F \) and \( B \).)

\subsubsection{Gradients through the Forward/Backward Difference Primitives}
\label{sec:mvm:gradients}

In this section we will let \( y = F(A, \dt, x) \) or \( y = B(A, \dt, x) \) denote the computation of interest,
\( L(y) \) denote a generic loss function,
and \( dx, dy, \dots \) denote gradients to \( x, y, \dots \) (e.g., \( dx = \frac{\partial L(y)}{\partial x} \)).

\paragraph{Derivatives of backward difference.}
First we have the standard \( \frac{\partial L(y)}{\partial x} = \frac{\partial L(y)}{\partial y} \frac{\partial y}{\partial x} = \frac{\partial L(y)}{\partial y} (I + \dt A)^{-1} \).
This corresponds to matrix-vector multiplication by \( (I + \Delta A)^{-T} \).
In other words, it can be computed by the primitive \( B(A^T, \dt, dy) \).

Similarly, in order to compute \( \frac{\partial L(y)}{\partial \dt} \) we require \( \frac{\partial y}{\partial \dt} \).
We need the result \( \frac{\partial Y^{-1}}{\partial x} = -Y^{-1} \frac{\partial Y}{\partial x}Y^{-1} \)for an invertible matrix \( Y \) \citep[equation (59)]{petersen2012matrix}.
Then
\begin{align*}
  \frac{\partial y}{\partial \dt}
  &=
  \frac{\partial (I + \dt A)^{-1}}{\partial \dt} x
  \\
  &=
  -(I + \dt A)^{-1} \frac{\partial (I + \dt A)}{\partial \dt } (I + \dt A)^{-1} x
  \\
  &=
  -(I + \dt A)^{-1} A (I + \dt A)^{-1} x
\end{align*}
and
\begin{align*}
  \frac{\partial L(y)}{\partial \dt}
  &=
  \frac{\partial L(y)}{\partial y} \frac{\partial y}{\partial \dt}
  \\
  &=
  -\left[ \frac{\partial L(y)}{\partial y} (I + \dt A)^{-1} \right] A \left[ (I + \dt A)^{-1} x \right]
\end{align*}

We can summarize this as follows. Let \( y = B(A, \dt, x) = (I + \dt A)^{-1} x \) and \( dy = \partial L(y) / \partial y \) (as a column vector).
Then
\begin{align*}
  y &= B(A, \dt, x) \\
  dx &= B(A^T, \dt, dy) \\ %
  d\dt &= -dx^T A y
  .
\end{align*}

\paragraph{Derivatives of forward difference.}
The forward case is simpler.
Let \( y = F(A, \dt, x) = (I + \dt A) x \).
Then \( \frac{\partial y}{\partial x} = I + \dt A \) and \( \frac{\partial y}{\partial \dt} = Ax \).
Thus
\begin{align*}
  y &= F(A, \dt, x) \\
  dx &= (I + \dt A)^T dy = F(A^T, \dt, dy) \\
  d\dt &= dy^T A x
  .
\end{align*}

\subsubsection{Computing the Forward/Backward Difference for Tridiagonal Inverse Matrices}
\label{sec:mvm:trid}

\cref{thm:trid} uses the classes of matrices \( A = P(D + T^{-1})Q \) for diagonal \( D, P, Q \) and tridiagonal \( T \).
We describe how the forward and backward difference MVMs can be performed efficiently for this class of matrices by reducing to a black-box tridiagonal solver.

\paragraph{Forward difference.}
It is straightforward to compute
\begin{align*}
  F(A, \dt, x) = (I + \dt \cdot P(D + T^{-1})Q) x = x + \dt \cdot PDQ x + \dt \cdot PT^{-1}Qx
\end{align*}
in terms of multiplication by diagonal matrices \( x \mapsto D x \) and tridiagonal solving \( x \mapsto T^{-1} x \).

\paragraph{Backward difference.}
We will explicitly rewrite the inverse of the matrix \( G = I + \dt \cdot P(D + T^{-1})Q \).

The core observation is to multiply \( G \) by a choice selection of matrices to cancel out the \( T^{-1} \) term:
\begin{align*}
  TP^{-1} G Q^{-1} = TP^{-1} Q^{-1} + \dt T D + \dt I.
\end{align*}
Rearranging yields
\begin{align*}
  G^{-1} = Q^{-1} (TP^{-1}Q^{-1} + \dt T D + \dt I)^{-1} T P^{-1}.
\end{align*}
Now note that the matrix in the middle is tridiagonal.
Hence we have reduced MVM by \( G^{-1} \), i.e. the backward difference problem, to a series of diagonal and tridiagonal MVMs (easy), and a tridiagonal inverse MVM (a.k.a. a tridiagonal solve).

\section{Additional Experiments and Experiment Details}
\label{sec:experiments-full}

We provide additional experiments and ablations in \cref{sec:experiments:additional}.
\cref{sec:experiments:methodology} describes our training methodology in more detail for each dataset.
The hyperparameters for all reported results are in \cref{tab::best-hyperparameters}.

\subsection{Additional Experiments}
\label{sec:experiments:additional}

\paragraph{Missing Data on CharacterTrajectories.}

\cref{tab:ct} has results for a setting considered in previous work involving irregularly-sampled time series.
LSSL is competitive with the best prior methods, some of which were specialized to handle this setting.

\begin{table}[!t]
    \caption{Test accuracies for irregularly sampled time series on the CharacterTrajectories dataset.
    \( p\% \) denotes percent of data that was randomly dropped.}
    \centering
    \begin{tabular}{@{}lllll@{}}
        \toprule
        Model                              & \( 0\% \)      & \( 30\% \)     & \( 50\% \)     & \( 70\% \)     \\
        \midrule
        GRU-ODE \citep{de2019gru}          & -              & 92.6           & 86.7           & 89.9           \\
        GRU-$\dt$ \citep{kidger2020neural} & -              & 93.6           & 91.3           & 90.4           \\
        GRU-D \citep{che2018recurrent}     & -              & 94.2           & 90.2           & 91.9           \\
        ODE-RNN \citep{rubanova2019latent} & -              & 95.4           & 96.0           & 95.3           \\
        NCDE \citep{kidger2020neural}      & -              & 98.7           & 98.8           & \textbf{98.6 } \\
        CKCNN \citep{romero2021ckconv}     & \textbf{99.53} & \textbf{98.83} & 98.60          & 98.14          \\
        LSSL                             & 99.30          & \textbf{98.83} & \textbf{98.83} & 98.37          \\
        \bottomrule
    \end{tabular}
    \label{tab:ct}
\end{table}

\paragraph{\( A \) and \( \dt \) ablations.}

\cref{tab:dt-cifar,tab:dt-sc} show results on SpeechCommands-Raw and a smaller model on sCIFAR,
ablating that learning either the \( A \) or \( \dt \) parameters provides a consistent performance increase.

\begin{table}[!t]
    \caption{\( A \) and \( \dt \) ablations on sCIFAR.}
    \centering
    \begin{tabular}{@{}lll@{}}
        \toprule
        & Learn \( \dt \) & Fixed \( \dt \) \\
        \midrule
        Learn \( A \) & 82.70 & 80.34 \\
        Fixed \( A \) & 80.61 & 80.18 \\
        \bottomrule
    \end{tabular}
    \label{tab:dt-cifar}
\end{table}

\begin{table}[!t]
    \caption{\( A \) and \( \dt \) ablations on SC-Raw.}
    \centering
    \begin{tabular}{@{}lll@{}}
        \toprule
        & Learn \( \dt \) & Fixed \( \dt \) \\
        \midrule
        Learn \( A \) & 96.07 & 95.20 \\
        Fixed \( A \) & 91.59 & 90.51 \\
        \bottomrule
    \end{tabular}
    \label{tab:dt-sc}
\end{table}

Finally, \cref{fig:dt-visualization} plots the \( \dt \) values at the beginning and end of training on the SpeechCommands-Raw dataset, confirming that training \( \dt \) does noticeably change their values to better model the data.
In particular, the \( \dt \) values spread over time to cover a larger range of timescales.

\begin{figure}[ht]
  \centering
  \begin{subfigure}[t]{0.5\linewidth}
    \centering
    \includegraphics[width=\textwidth]{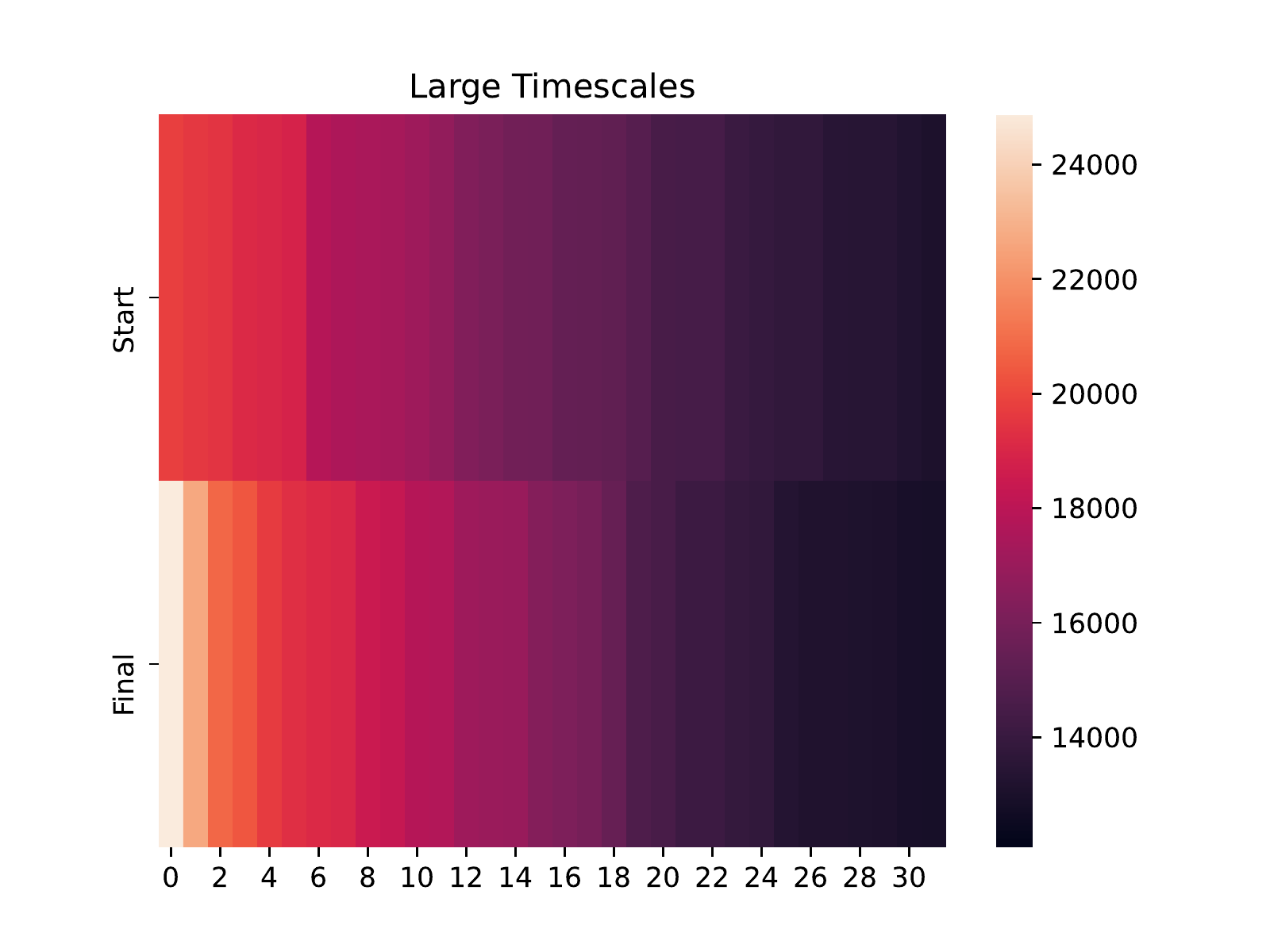}
  \end{subfigure}
  \begin{subfigure}[t]{0.5\linewidth}
    \centering
    \includegraphics[width=\textwidth]{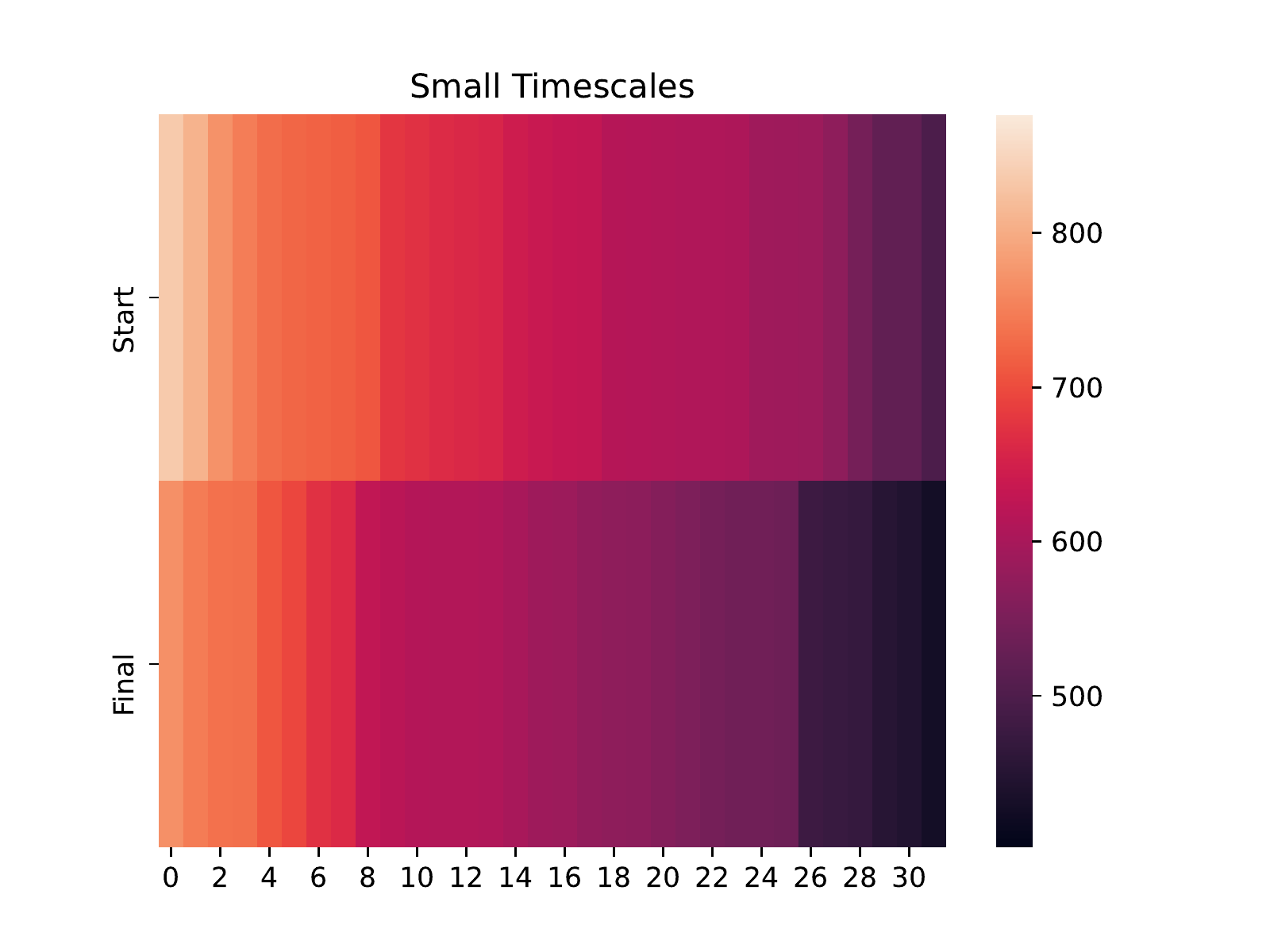}
  \end{subfigure}
  \caption{
    We visualize the 32 largest and smallest \( \dt \) values at the start and end of training for the first layer of our state-of-the-art LSSL model on the Speech Commands Raw dataset.
    The plots visualize $\frac{1}{\dt}$, which can be interpreted as the timescale at which they operate (\cref{sec:background}).
    The plots confirm that LSSL does modify the dt values in order to more appropriately model the speech data.
  }
  \label{fig:dt-visualization}
\end{figure}

\subsection{Methodology}
\label{sec:experiments:methodology}

We describe our training procedure on each dataset for our model and any relevant baselines.

\paragraph{General}
All models and datasets used the Adam optimizer with a LR decay scheduler that reduced LR by 5x upon validation plateau for 10 or 20 epochs.
We fixed the batch size to \( 50 \) for the MNIST/CIFAR datasets and \( 32 \) for other datasets, reducing if necessary to fit in memory.

For all models, we chose the hyperparameters that achieved the highest validation accuracy/RMSE (values in Table \ref{tab::best-hyperparameters}).

\paragraph{Error Bars}
We note that the results in \cref{sec:experiments} do not include standard deviations for formatting reasons, since most of the baselines were best results reported in previous papers without error bars.
As \cref{sec:discussion} noted, the LSSL was actually quite stable in performance and not particularly sensitive to hyperparameters.
We note that for every result in \cref{sec:experiments}, the LSSL with error bars was at least one standard deviation above the baseline results.

\subsubsection{Sequential and Permuted MNIST}

The model architecture of LSSL(-f) was fixed to the small architecture with 200K parameters (\cref{sec:model}).
Following \citep{romero2021ckconv}, we fixed the learning rate scheduler to decay on plateau by with a factor of \( 0.2 \), and the number of epochs to 200.
We searched hyperparameters over the product of the following learning rate values: $\{0.001, 0.002, 0.004, 0.01\}$, and dropout values: $\{0.1, 0.2\}$.

\subsubsection{Sequential CIFAR}

The model architecture of LSSL(-f) was fixed to the large architecture with 2M parameters (\cref{sec:model}).
We searched over the product of the following learning rate values: $\{0.001, 0.002, 0.004, 0.01, 0.02\}$, and dropout values: $\{0.2, 0.3, 0.4\}$.

\subsubsection{BIDMC Healthcare}

The BIDMC tasks aim at predicting three vital signs of a patient, respiratory rate (RR), heart rate (HR), and oxygen saturation (SpO2), based on PPG and ECG signals.
The clinical data is provided by the Beth Israel Deaconess Medical Center.
The PPG and ECG signals were sampled at $125$Hz and have a sequence length of $4000$.

For this dataset, we fixed the small LSSL(-f) model (\cref{sec:model}).
Following \citep{rusch2021unicornn}, we changed the scheduler to a multistep scheduler that decays on fixed epochs, and trained for 500 epochs.

For our methods, we searched over the product of the following learning rate values: $\{0.004, 0.01, 0.02\}$, and dropout values: $\{0.1, 0.2\}$.

\paragraph{Baseline parameters.}
For CKConv, we searched over \( \omega_0 \in [10, 50] \) following the guidelines of \citet{romero2021ckconv} (best value \( \omega_0=20 \)).
Since we tuned the sensitive \( \omega_0 \), we fixed the learning rate to $0.001$ and dropout to $0.1$ which was the default used in \citep{romero2021ckconv}.

The transformer model we used was a vanilla transformer with a hidden dimension of $256$, $8$ attention heads, $4$ layers, and a feedforward dimension of $1024$.
We used a learning rate of $0.001$ and a dropout of $0$.
We tried a few variants, but no transformer model was effective at all.

\subsubsection{CelebA}

For these larger datasets, we reduced the size of the order \( N \) and did not tie it to \( H \).
These experiments were computationally heavy and we did not do any tuning (i.e., \cref{tab::best-hyperparameters} are the only runs).
The model size was picked to train in a reasonable amount of time,
and the learning rate for the first attribute was picked based on general best hyperparameters for other datasets,
and then reduced for subsequent experiments on the other attributes.

\textbf{Baseline parameters.}
For ResNet-18, we used the standard implementation with a learning rate of $0.001$.

\subsubsection{Speech Commands}
For Speech Commands, we use the same dataset and preprocessing code from \citet{kidger2020neural,romero2021ckconv}.
We consider the two settings from \citet{kidger2020neural}:
SC-Raw uses very long time-series raw speech signals of 16000 timesteps each,
while SC-MFCC uses standard MFCC features of 161 timesteps.

For our models trained over the raw data, we searched over the product of the following learning rate values: $\{0.002, 0.004, 0.01\}$, and dropout values: $\{0.1, 0.2\}$.
For our models trained over the MFCC features, we searched over the product of the following learning rate values: $\{0.0001, 0.001, 0.002, 0.004, 0.01\}$, and dropout values: $\{0.1, 0.2, 0.3, 0.4\}$.

\textbf{Baseline parameters.}
To get more results for the strongest baselines on very long sequences in the literature,
we ran the UniCORNN~\citep{rusch2021unicornn} baseline on both Raw and MFCC variants,
and the Neural Rough Differential Equations~\citep{morrill2021neural} baseline on the Raw variant.

For UniCORNN trained over the raw data, we searched over multiple hyperparameters.
Specifically, we searched over alpha: $\{0, 10, 20, 30, 40\}$, \( \dt \) values: $\{0.00001, 0.0001, 0.001, 0.01\}$, and learning rate values: $\{0.0001, 0.0004, 0.001, 0.004\}$.
However, since the method was not able to generalize to the validation set for any hyperparameter combination, we used the authors' reported hyperparameters for the Eigenworms dataset as it also contains very long sequences (\( \approx 18000 \)).
In particular, we used a learning rate of $0.02$, hidden dimension of $256$, $3$ layers with dt values $[0.0000281, 0.0343, 0.0343]$, dropout of $0.1$, and alpha of $0$.

For UniCORNN trained over the MFCC features, we used the authors' reported hyperparameters for the MNIST dataset (again due to similarly sized sequence lengths), and further tuned the learning rate over the values: $\{0.0001, 0.001, 0.005, 0.01, 0.02\}$, \( \dt \) values: $\{0.01, 0.1\}$, and alpha values: $\{10, 20, 30\}$.

The best model used a learning rate of $0.02$, hidden dimension of $256$, $3$ layers with dt values of $0.19$, dropout of $0.1$, and alpha of $30.65$.

For NRDE on SC-Raw, we used depth $2$, step size $4$, hidden dimension $32$, and $3$ layers.
Our results were better than unofficial numbers reported in correspondence with the authors,
so we did not tune further.

\subsubsection{Convergence Speed \texorpdfstring{(\cref{tab:speed-training})}{}}

The convergence table compared against logs directly from the corresponding baseline's SoTA models \citep{rusch2021unicornn,romero2021ckconv}, which were either released publicly or found in direct correspondence with the authors.
To generate the wall clock numbers, we ran the baseline models on the same hardware as our models and extrapolated to the target epoch.

\subsection{Hyperparameters}

Best hyperparameters for all datasets are reported in \cref{tab::best-hyperparameters}.

\begin{table*}[!t]
  \caption{The values of the best hyperparameters found for each dataset.}
  \label{tab::best-hyperparameters}
  \centering
  \resizebox{\textwidth}{!}{%
    \begin{tabular}{@{}lllllllllll@{}}
      \toprule
      \multirow{1}{*}{\bf Dataset} & \multicolumn{7}{c}{\bf Hyperparameters} \\
      \midrule
                                   & {\bf Learning Rate}                      & {\bf Dropout} & {\bf Batch Size} & {\bf Epochs} & {\bf Depth} & {\bf Hidden Size} \( H \) & {\bf Order} \( N \) & {\bf Channels \( M \)}  &  \\
      \cmidrule(l){2-9}
      {\bf sMNIST}                 & $0.004$                                  & $0.2$         & $50$             & $200$        & $6$         & $128$                     & $128$               & $1$                    \\
      {\bf pMNIST}                 & $0.001$                                  & $0.2$         & $50$             & $200$        & $6$         & $128$                     & $128$               & $1$                    \\
      {\bf sCIFAR}                 & $0.02$                                   & $0.3$         & $50$             & $200$        & $4$         & $256$                     & $256$               & $4$                    \\
      {\bf BIDMC-RR}               & $0.004$                                  & $0.1$         & $32$             & $500$        & $6$         & $128$                     & $128$               & $1$                    \\
      {\bf BIDMC-HR}               & $0.01$                                   & $0.2$         & $32$             & $500$        & $6$         & $128$                     & $128$               & $1$                    \\
      {\bf BIDMC-SpO2}             & $0.01$                                   & $0.1$         & $32$             & $500$        & $6$         & $128$                     & $128$               & $1$                    \\
      {\bf SC Raw}                 & $0.01$                                   & $0.2$         & $16$             & $50$         & $4$         & $256$                     & $128$               & $2$                    \\
      {\bf SC MFCC}                & $0.004$                                  & $0.4$         & $32$             & $100$        & $6$         & $128$                     & $128$               & $1$                    \\
      {\bf sCelebA-Att.}           & $0.002$                                  & $0.1$         & $32$             & $200$        & $3$         & $256$                     & $128$               & $4$                    \\
      {\bf sCelebA-MSO}            & $0.002$                                  & $0.1$         & $32$             & $200$        & $3$         & $256$                     & $128$               & $4$                    \\
      {\bf sCelebA-Smil.}          & $0.01$                                   & $0.1$         & $32$             & $200$        & $3$         & $256$                     & $128$               & $4$                   \\
      {\bf sCelebA-WL}             & $0.002$                                  & $0.1$         & $32$             & $200$        & $3$         & $256$                     & $128$               & $4$                    \\
      \bottomrule
    \end{tabular}%
  }
\end{table*}

\end{document}